\documentclass[a4paper,fleqn]{cas-dc}

\usepackage[authoryear,longnamesfirst]{natbib}

\def\tsc#1{\csdef{#1}{\textsc{\lowercase{#1}}\xspace}}
\tsc{WGM}
\tsc{QE}

\usepackage{microtype}
\usepackage{graphicx}
\usepackage{subcaption}
\usepackage{booktabs}
\usepackage{multirow}
\usepackage{amsmath}
\usepackage{amssymb}
\usepackage{mathtools}
\usepackage{amsthm}

\newtheorem{theorem}{Theorem}  
\newtheorem{proposition}[theorem]{Proposition}

\newtheorem{corollary}[theorem]{Corollary}
\theoremstyle{definition}

\newdefinition{rmk}{Remark}
\newtheorem{remark}[theorem]{Remark}

\usepackage{xcolor}

\begin{document}
\let\WriteBookmarks\relax
\def\floatpagepagefraction{1}
\def\textpagefraction{.001}

\shorttitle{Spectral alignment for time series latent flows}    

\shortauthors{Carvajal Reyes and Tobar}  

\title[mode = title]{Time series generation with spectrally aligned latent flow matching}  



\author[1]{Camilo Carvajal~Reyes}[orcid=0009-0000-8334-1668]
\cormark[1]
\ead{c.carvajal-reyes24@imperial.ac.uk}
\ead[url]{}

\credit{Writing - original draft, Writing: review \& editing, Methodology, Conceptualization, Formal analysis, Visualisation}

\affiliation[]{organization={Department of Mathematics, Imperial College London},
    addressline={180 Queen's Gate}, 
    city={London},
    citysep={}, 
    postcode={SW7 2HR}, 
    state={},
    country={United Kingdom}}

\author[2]{Felipe Tobar}[orcid=0000-0003-2486-3583]
\ead{f.tobar@imperial.ac.uk}
\ead[url]{}

\credit{Writing: review \& editing, Writing, Validation, Formal analysis, Supervision}

\cortext[1]{Corresponding author}



\begin{abstract}
Latent flow models have proven to be a reliable and cost-effective method for time series generation.
However, the latent compression induces unwanted artefacts, such as a spectral mismatch with respect to the underlying dataset, thus hindering their use as training surrogates.
In this article, we propose a spectrally-aligned latent-flow time series generator, where the latent space for flow matching is trained to preserve dynamical properties that are relevant for the suitability of synthetic samples. We find that incorporating fine-tuning losses based on canonical signal representations such as the Fourier, wavelet and signature transforms helps overcome these issues.
The interpretability of these transformations allows us to ensure that the synthetic signals are aligned with the true ones in terms of relevant features, such as smoothness or targeted spectral content, as opposed to relying on pointwise reconstruction losses only.
We compare the proposed aligned models against a base latent-flow model and the state of the art over real-world long-range univariate and multivariate benchmark datasets. Our quantitative results validate the superiority of the proposed method in terms of its performance on metrics reflecting signal realness and computational efficiency, while being aligned to the training set with respect to its local structure.
\end{abstract}


\maketitle


\section{Introduction}

Probabilistic generative modelling for time series (TS) aims to synthesise signals following a reference distribution which is only available through observed samples. Under the unconditional generation setting, the synthetic series can be used for data augmentation or as training surrogates for large-scale downstream tasks. Upon conditioning, the generating distribution can be leveraged for canonical time series tasks such as imputation, forecasting or denoising.

Recent probabilistic models for multidimensional time series build on two related perspectives \citep{lipman_flow_2022}. First, diffusion models \citep{sohl-dickstein_deep_2015,song_generative_2019,ho_denoising_2020}, which generate series via denoising a Gaussian noise source, incurring a high sampling cost. Second, flow-matching models \citep{lipman_flow_2022}, which train a velocity field that transports from any suitable source distribution (e.g., Gaussian) towards a desired target distribution. In real-world applications involving large-dimensional data such as time series, both diffusion- and flow-based models may operate on latent spaces \citep{rombach_high-resolution_2022,dao_flow_2023} as a means to bound computational complexity. Naturally, this latent representation can be seen as a form of lossy compression, where achieving the desired computational improvement comes at the cost of a degradation in the quality of the synthetic samples. See Fig.~\ref{fig:freq_profile} for an illustration of these artefacts.

\begin{figure}[]
    \centering
    \includegraphics[width=\linewidth]{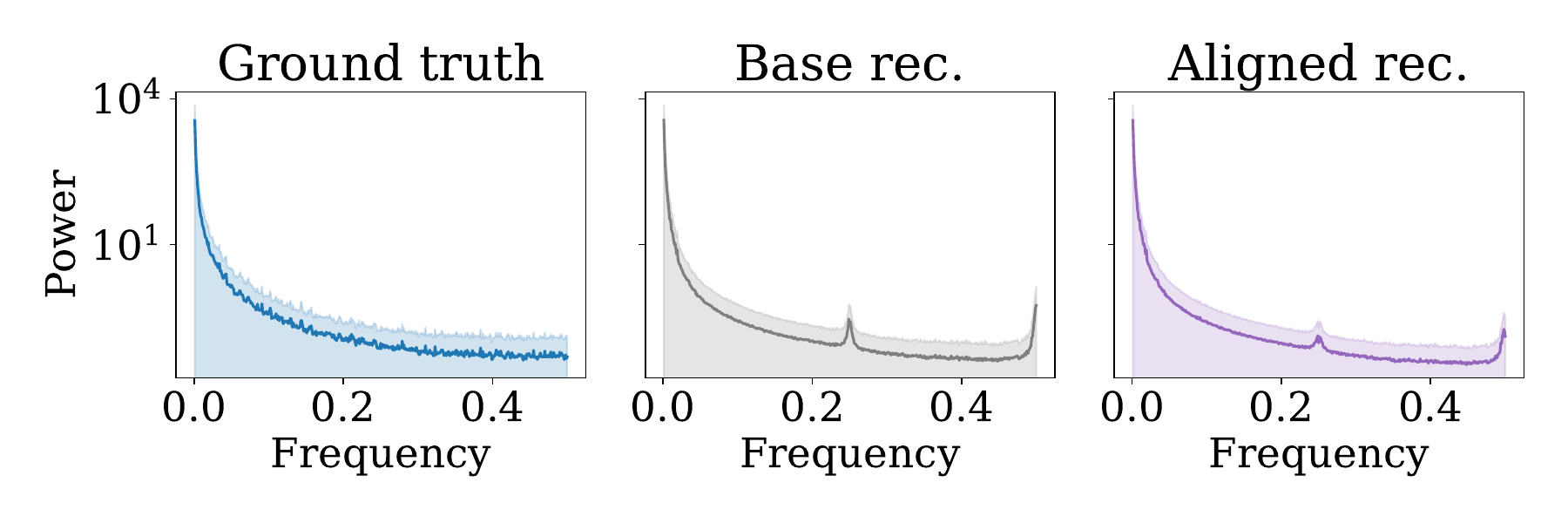}
    \caption{Fourier power spectrum of the Exchange rate dataset \citep{lai_modeling_2018} (mean $\pm$ 1 st.~dev.~, logscale): true samples (left), example of a non-aligned base flow model (centre), and our proposed aligned flow model (right). Observe how the non-aligned model exhibits a spectral artefact in the form of a peak at around frequency $0.25$, which is largely minimised by the proposed alignment strategy.}
    \label{fig:freq_profile}
\end{figure}

Hence, an effective latent flow method to sample realistic time series implies addressing these artefacts. These are a consequence of the spectral bias, which postulates that lower frequencies are prioritised when training neural networks \citep{rahaman_spectral_2019}.
We hypothesise that the latent space can be designed to reduce the detrimental effect that the latent compression has on the relevant dynamical features of the synthetic signals, such as the signal regularity, e.g., the smoothness of the underlying continuous path of the time series. The key element for this is using signal representations allowing us to identify those key features in order to shift the optimisation weights accordingly. In other words, given that the compressed representation can be interpreted as a budget constraint, we seek to allocate the computational budget to features that matter most.

We restrict ourselves to flow-based models with autoencoder-induced latent spaces, and implement the above rationale by training the encoder-decoder pair through a novel set of transform-consistency losses that promote the preservation of high-frequency features, directly in the latent space. This is implemented as a fine-tuning step that we refer to as ``alignment''. Our losses are based on different spectral representations of the series, leveraging the complementary properties of the Fourier, wavelet and signature transforms. Intrinsically, this methodology aims to guide the compression of the series to maintain the spectral content that determines regularity, rather than only minimising the Euclidean discrepancy between the true and reconstructed samples as customary in the design of the latent space. We provide an illustration of our pipeline in Figure \ref{fig:latent-flow-matching}. \\

\noindent The main contributions of this work are:
\begin{enumerate}
    \item The implementation of a base latent-flow model using an encoder-decoder pair trained with the Euclidean loss that performs well in temporal metrics at a remarkable reduction in sampling cost. This benchmarking of latent flow matching for TS has been largely unexplored and serves as motivation to tackle the issues arising from this compression.  
    \item A performance analysis of the base latent space, showing that the share of the reconstruction error concentrates in high-detail components. This indicates that the standard training objective under-penalises them, as a consequence of the well-known spectral bias \citep{rahaman_spectral_2019}, which undermines the samples realism.
    \item A family of loss functions, namely the transform-consistency losses, that promote the preservation of relevant spectral features in the latent representation. The flexibility of the losses allows us to combine them with a weighting scheme that prioritises higher-order features. The losses are then used to improve the latent space via fine-tuning. To the best of our knowledge, the use of general time series transforms to induce geometric structure in the context of generative modelling has not been studied.  
    \item Theoretical guarantees that our weighted transform-consistency losses can be regarded as matching the roughness of the input signal in the Fourier case, as assessed by the discrete Sobolev norm; and that the Signature-consistency loss can be regarded as an approximation of the norm corresponding to a signature kernel.
    \item Empirical quantitative evidence that fine-tuning the encoder-decoder pair with the transform-consistency losses improves the quality of the generated signals in terms of the discriminative score (which measures how difficult it is to distinguish real signals from synthetic ones), while maintaining the computational gain of the base flow-based model, as well as their performance in waveform-based metrics.
\end{enumerate}

The rest of the article is organised as follows: Sec.~\ref{sec:background} presents flow models, latent-space models and spectral representations; Sec.~\ref{sec:base_model} establishes the general latent flow matching procedure;  Sec.~\ref{sec:methodology} presents our proposed alignment method addressing the drawbacks of latent models, while Sec.~\ref{sec:norm_equivalences} shows the theoretical connections of the proposed losses with other discrepancies. The proposed approach is empirically validated in Section \ref{sec:experiments}. We highlight recent work on autoencoders and flow-based TS generation in Section \ref{sec:related_work}, and an analysis of the nature of the improvements in Section \ref{sec:analysis}, along with the conclusions and future work in Sec \ref{sec:conclusion}.

\begin{figure}[]
  \centering
  \includegraphics[width=0.9\linewidth]{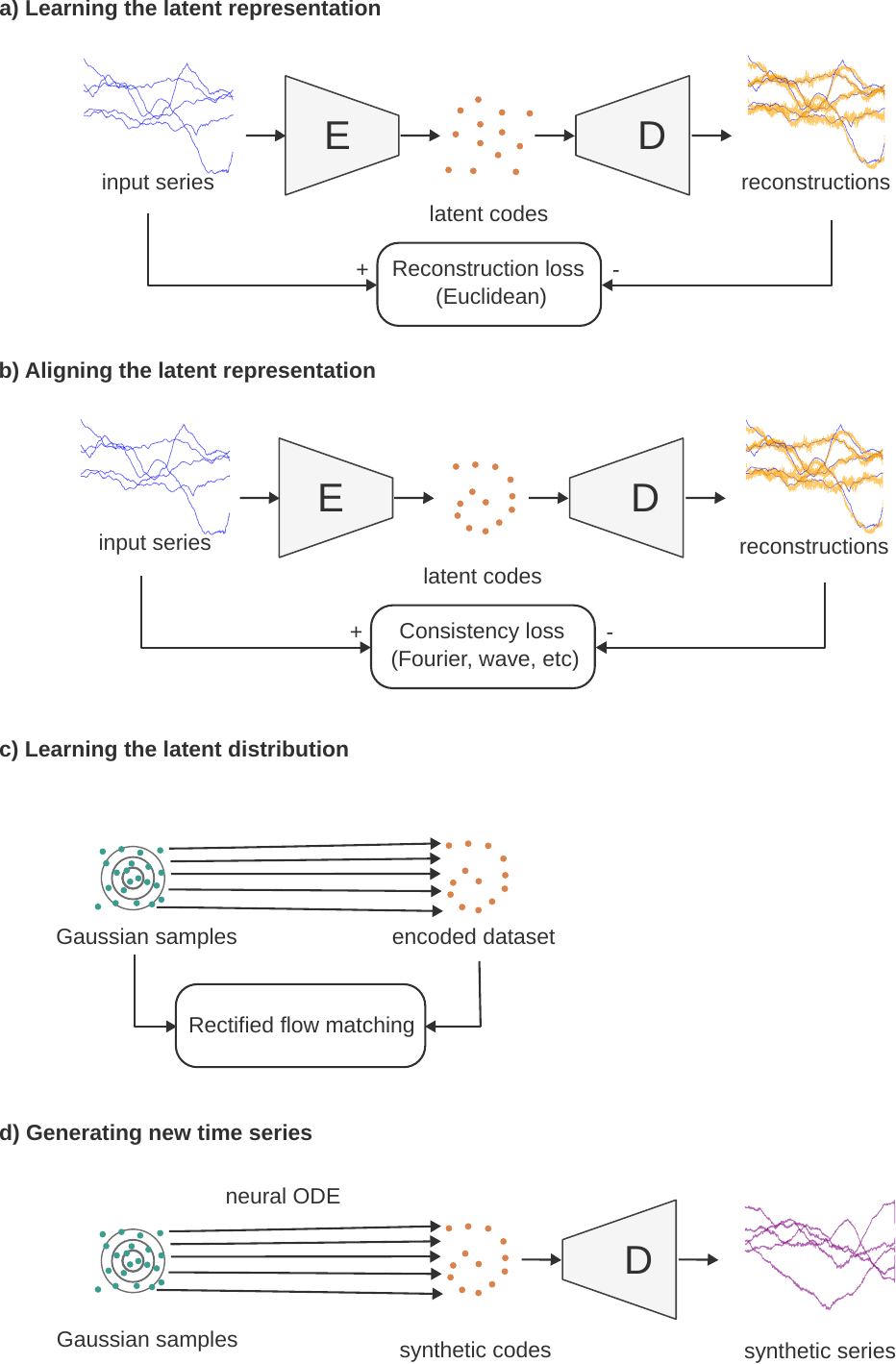}
  \caption{\textbf{Latent flow matching for time-series generation.}
    We a) train an encoder-decoder pair $(\mathcal{E},\mathcal{D})$ for reconstruction, then b) fine-tune it with a \emph{transform-consistency} loss.  We then c) train a flow-matching model in the latent space to transport a Gaussian prior to the distribution of encoded latents, and d) sample by drawing from the prior, applying the flow, and decoding to obtain new time-series samples.}
    \label{fig:latent-flow-matching}
\end{figure}

\section{Background}
\label{sec:background}

\subsection{Flow models}
\label{sec:FM}

Given datapoints $\{x^{(i)}\}_{i=1}^N\subset\mathbb{R}^d$, flow matching seeks to sample from the data distribution $p_{\text{data}}$ by defining a transport from a source (e.g., Gaussian) distribution $p_0$ to the target distribution $p_1 \approx p_{\text{data}}$. In this context, a collection of intermediate probability densities $\{p_t\}_{t\in[0,1]}$ is defined by applying the pushforward operator\footnote{equivalently, $\phi_t(x)\sim p_t$ for $x \sim p_0$.} $p_t = [\phi_t]_\# p_0$. The map $\phi_t(\cdot): [0,1]\times \mathbb{R}^d \to \mathbb{R}^d$ is defined as the solution to the ordinary differential equation (ODE) \citep{lipman_flow_2022}:
\begin{equation}
\label{eq:ODE_FM}
	\frac{d}{dt} \phi_t (x) = u_t(\phi_t(x)); \quad \phi_0(x) = x,
\end{equation}
where $u_t:\mathbb{R}^d \to \mathbb{R}^d$, $t\in[0,1]$ is a time-dependent vector field. The corresponding continuous normalising flow, or simply ``flow'', $\phi_t$ returns the position of $x$ at time $t$ when following the vector field $u_t(\cdot)$. Furthermore, $\phi_t$ is said to be constructed by the velocity field $u_t$ when eq.~\eqref{eq:ODE_FM} is satisfied.
In practice, a conditional version of $u_t$, expressed in terms of the samples from source and target distributions, is approximated for general $x_t$ with a neural network $u_\theta(x_t,t)$ \citep{lipman_flow_2022,tong_improving_2023} via regression.  
The model $u_\theta$ is required to be continuously differentiable with bounded derivatives, a condition that can be easily satisfied when using neural networks \citep{holderrieth_introduction_2025} with activation functions such as ReLU. This ensures the existence and uniqueness of a flow $\phi_t$ solving eq.~\eqref{eq:ODE_FM} \citep{perko_differential_2013,coddington_theory_1956}.
After training, sampling is achieved by drawing $x_0\sim p_0$ and then solving the ODE in eq.~\eqref{eq:ODE_FM} with $x_0$ as the initial value. In practice, the flow model is implemented such that the intermediate probability density follows a Gaussian transition, i.e., 
\[ p_t(x\mid x_0, x_1) = \mathcal{N}(x;\mu(x_0, x_1,t), \sigma(x_0, x_1,t)^2 I)\,, \]
where the mean is the linear interpolation between source and target (data) samples $x_0$ and $x_1$, respectively \citep{tong_improving_2023}. The resulting conditional velocity field reduces to  $ x_1-x_0$, as a consequence of this choice of mean. Subsequently, the velocity field model $u_\theta$ is trained to match this straight line in the ambient space. Hence, we obtain the following conditional flow-matching objective:
\[ \mathcal{L} = \mathbb{E}_{t\sim \mathcal{U}([0,1]), x_1\sim p_1, x_0\sim p_0,x_t\sim p_t(\cdot|x_1,x_0)} \| u_\theta(x_t,t) - (x_1 - x_0) \|^2 \,. \]
The limit case where $\sigma \equiv 0$, namely rectified flow \citep{liu_flow_2022}, has been widely adopted in practice.

\begin{rmk}
A flow given by a straight trajectory in the Euclidean space would not necessarily be a straight transport with respect to the signal in other representations. Therefore, we propose embedding a latent space with dynamics modelled by suitable TS transforms, allowing for a flow that would be trained by linearly transporting the source distribution to a (latent) target one that does reflect these signal representations.
\end{rmk}

\subsection{Latent spaces for scaling flow models}
\label{sec:latent_spaces}

Flow models have benefited from training and sampling in a lower-dimensional latent space, notably leading to faster sampling than models defined directly in the ambient space \citep{rombach_high-resolution_2022}. The learning strategy is separated into two stages: (1) learning the latent space by training an encoder-decoder structure (also referred to as the autoencoder, AE), and (2) defining the flow model in this lower-dimensional space. The compression model introduces latent variables $z \in \mathbb{R}^l$, $l \ll d$, via an encoder $\mathcal{E} : \mathbb{R}^d \rightarrow \mathbb{R}^l$ and a decoder $\mathcal{D} : \mathbb{R}^l \rightarrow \mathbb{R}^d$. These are trained by minimising a reconstruction error $\mathcal{L}_{\mathrm{recon}} = -\mathbb{E}_{z \sim q_\phi(z|x)} [\log p_\theta(x|z)]$, as well as a weighted adversarial term. The latter corresponds to the loss of a discriminator that is trained to distinguish between synthetic and real inputs \citep{goodfellow_generative_2014}, and it has been adopted in the context of image synthesis \citep{esser_taming_2021,yu_vectorquantized_2022}.
In this setting, which is similar to adversarial autoencoders \citep{makhzani_adversarial_2015}, the adversarial loss effectively acts as regularisation \citep{lamb_discriminative_2016}. This has proven to be successful in regularising latent spaces for transformer-based probabilistic forecasting \citep{zhang_latent_2022}. 
This way, denoting $\mathcal{L}_{\mathrm{Adv}}= \log D(x) + \log(1-D(\bar{x}))$, where $D$ is the discriminator network, and a weighting factor $0 < \lambda < 1$, the loss can be written as follows:
\begin{equation}
\label{eq:VAE_loss}
\mathcal{L}_{\mathrm{AE}} = \mathcal{L}_{\mathrm{recon}} + \lambda \mathcal{L}_{\mathrm{Adv}}.
\end{equation}

\subsection{Time series representations}
\label{sec:background_ts}
 
A single time series will be represented by $x\in \mathbb{R}^{m \times d}$, that is, $m$ (uniform in time) realisations of an $\mathbb{R}^d$-valued stochastic process.
The $(m \times d)$-dimensional vector will be referred to as the time representation of the time series (TS), also called signal or path.

\paragraph{Fourier transform.}

For a $1$-dimensional series, the discrete Fourier transform is defined as $\mathcal{F}:\mathbb{R}^{d \times m}\to\mathbb{C}^{d \times m}$,
\begin{equation}
\label{eq:DFT}
    \mathcal{F}[x](\xi) = \frac{1}{m} \sum^{m-1}_{t=0} x_t e^{-\xi 2\pi it} \,,
\end{equation}
where $i$ is the imaginary unit.
Key properties of the Fourier transform  $\mathcal{F}$ are linearity and symmetry. The latter is exploited by the Fast Fourier Transform algorithm (FFT) \citep{cooley_algorithm_1965}, which is the most widely used method to compute discrete frequency samples from the time representation of a TS.

\paragraph{The wavelet transform.}

Decompositions of TS based on Fourier coefficients are able to capture precise frequency values, but present low sensitivity in the time domain. This is called \textit{frequency localisation}. The converse is called \textit{time localisation}, and the trade-off between both is known as the uncertainty principle \citep{gabor_theory_1946}. The wavelet decomposition \citep{meyer_wavelets_1992} is a method that presents moderately accurate localisation in both time and frequency domains. In this setting, a reconstruction associated with the discrete wavelet transform (DWT) at decomposition level $K$ represents a time series as:
\[ x(t) = A_K + \sum^{K}_{j=1} D_j = \sum_k cA_K[k] \phi_{K,k}(t) + \sum_{j,k} cD_j[k] \psi_{j,k}(t)\,, \]
where $\psi_{j,k} = 2^{j/2} \psi(2^j t-k)$ are derived from a mother wavelet $\psi(t)$ by shifting and scaling (indices $k$ and $j$ respectively), and likewise for $\phi_{K,k}$. The term $A_K$ corresponds to an approximation coefficient, i.e., modelling coarse aspects of the TS, equivalent to applying a low-pass filter. On the other hand, $\sum_{j=1}^K D_j$ is a sum of detail coefficients, computed recursively by applying a band-pass filter over many levels with downsampling \citep{daubechies_orthonormal_1998}. The coefficients $A_K$ and $\{D_j\}_1^{K-1}$ are calculated via the dot product between $x$ and the wavelets $\psi$ and $\phi_{j}$, respectively. The resulting orthonormal series is called the wavelet transform, and it fully determines the time series $x$. Crucially, the magnitude of wavelet coefficients determines the regularity of the underlying path, similarly to Fourier coefficients \citep{mallat_wavelet_1999}. We refer to Appendix \ref{sec:complements_TS} for a brief overview of wavelet families.

\paragraph{Signature Transform (ST).}
Signatures are representations of paths, whose motivation comes from the method of Picard iterations to obtain solutions to controlled differential equations \citep{cass_lecture_2024}. For paths defined over $[0,1]$ with bounded $p$-variation (see Appendix \ref{sec:complements_TS} for a more formal presentation), the $k$-fold iterated integral $S(x)^{(k)}$ is denoted by
\[ S(x)^{(k)} = \int_{0 < t_1 < t_2 < \ldots < t_k < 1} dx_{t_1} \otimes dx_{t_2} \otimes \ldots \otimes dx_{t_k} \,. \]
The signature transform is then the collection $ S(x) = \left(1, S(x)^{(1)}, \ldots, S(x)^{(k)}, \ldots \right)$. The signature can be seen as a basis for functions over curve spaces, rather than being a linear decomposition in a pre-existing basis, which is the case for the Fourier and wavelet transforms \citep{kidger_deep_2019}.

\section{Latent flow matching for TS generation}
\label{sec:base_model}

Recall that our goal is to generate high-quality time series that avoid the spectral bias, while maintaining the computational advantages of latent generative models. To this end, we propose a four-stage procedure, outlined in Figure~\ref{fig:latent-flow-matching} and presented as follows.

\textbf{I. Encoder-decoder training}:
An encoder-decoder pair $(\mathcal{E},\mathcal{D})$ is defined according to Section~\ref{sec:latent_spaces}. This compression model without fine-tuning will be referred to as the \textbf{base model}.

\textbf{II. Alignment of the latent space}:
The autoencoder is then fine-tuned by incorporating an additional loss, namely the transform-consistency loss, thereby improving the smoothness by allocating more computational resources to higher-order features. We refer to this process as the \textbf{alignment}. The strategy we adopt for this step is presented in Section \ref{sec:methodology}.
    
\textbf{III. Latent flow matching}: A  (rectified) flow model is trained as outlined in Section \ref{sec:FM}, approximating the constant velocity $z_1-z_0$ from the interpolation $z_t = t z_1 + (1-t) z_0$, where $z_1$ and $z_0$ denote an encoded data point and a latent prior sample. That is, $u_\theta(z_t,t) \approx z_1 - z_0$, where $z_0$ is the realisation of an isotropic Gaussian. The target $z_1$ will correspond to signals represented via our aligned latent space.

\textbf{IV. Sampling}:
Samples are generated by the ODE in Equation \ref{eq:ODE_FM} by replacing $u_t$ with the trained network while using an isotropic Gaussian realisation $z_0$ as the initial value. This yields latent samples that are close to the encoded data. Synthetically generated time series are the result of decoding these back to the ambient space.

\section{Transform-consistency losses for spectral alignment}
\label{sec:methodology}

In addition to their computational advantage, the flexibility of latent spaces presents the opportunity to incorporate geometries beyond the Euclidean distance in its definition. Therefore, by employing time series transforms that enable us to prioritise high frequencies in the loss computation, we propose a means of counteracting the spectral bias exhibited by Euclidean-only latent spaces. To this end, we define a transform-consistency loss in its general form in Subsection~\ref{sec:general_loss}, followed by specific variants based on the Fourier, Wavelet and Signature transforms (Subsections~\ref{sec:fourier_loss},~\ref{sec:wavelet_loss} and~\ref{sec:signature_loss} respectively). These losses will be used to align (i.e., fine-tune) the base latent space for flow matching. Crucially, the definition of the transform-consistency loss will allow us to prioritise higher-order features, which is key to compensating for the spectral bias \cite{rahaman_spectral_2019}, which postulates that low frequencies are learnt first in neural network training. This priorisation of higher levels will be done via a Sobolev-inspired weighting scheme $w_k$, defined in Section~\ref{sec:norm_equivalences}.

\subsection{General transform-consistency loss}
\label{sec:general_loss}

Let $T : \mathbb{R}^{d \times m} \rightarrow \mathbb{R}^L$ be a time series transform, i.e., a mapping as those introduced in Section \ref{sec:background_ts}. We will consider transforms allowing us to model various levels of detail from the TS. Consequently, we assume that $T$ can be decomposed as 
\[
\begin{aligned}
T_1 &: \mathbb{R}^{d \times m} \rightarrow \mathbb{R}^{l_1} \\
&\vdots \\
T_k &: \mathbb{R}^{d \times m} \rightarrow \mathbb{R}^{l_k}
\end{aligned}
\]
with $L = l_1 + \dots + l_k$. We can then consider a weighted loss of the form:
\[ \mathcal{L} = \sum_{k=1}^K w_k^2\| T_k(x) - T_k(\bar{x}) \|^2 \,, \]
for some appropriate norm $\| \cdot \|$ and weighting scheme $w_k,\,k=1,...,K$. Here $\bar{x} = \mathcal{D}(\mathcal{E}(x))$ is the reconstruction of $x$ through the autoencoder. This general, transform-dependent evaluation will be referred to as the transform-consistency loss, since it assesses how close a reconstruction is to the original signal according to features computed from time series representations beyond the waveform, such as higher frequencies or decomposition levels.

The structure of the transforms considered in this paper is such that higher levels $k$ represent finer levels of detail. Consequently, we adopt a weighting scheme that will increase the weight as $k$ increases, hence prioritising local structure over global structure when adjusting the latent space. We refer to Section \ref{sec:sobolev} to present the scheme and its related theoretical notions, and to Section \ref{sec:high_freq_role} for an empirical view on the effect of upweighting higher levels of detail.

\subsection{Fourier-consistency loss}
\label{sec:fourier_loss}

Let $\mathcal{F}$ be the discrete Fourier transform, as defined in \eqref{eq:DFT}, we can define the \textit{Fourier-consistency loss} as
\begin{equation}
\label{eq:fourier_loss}
    \mathcal{L}_{\mathcal{F}} = \sum_{k=1}^{K} \sum_{\xi \in B_k} w_k^2 \left| F_\xi(x)-F_\xi(\bar x) \right|^2
\end{equation}
with $|\cdot|^2$ the square complex modulus and $B_k$, $k=1,\dots,K$ denotes a partition of the frequency grid.

\subsection{Wavelet-consistency loss}
\label{sec:wavelet_loss}

Let $A(x)$ denote the linear projection arising from projecting $x$ into the basis function $\phi$, as defined in Section \ref{sec:background_ts}. Similarly, let $D_k(x)$ be given by projecting $x$ onto $\psi_k$, for $ k=1,\dots,K$. Then weighted wavelet-consistency loss may be expressed as:

\begin{align*}
    \mathcal{L}_W(x, \bar{x}) = & w_{K+1}^2 \left\|A(x) - A(\bar{x}) \right\|^2 \\
    & + \sum_{k=1}^{K} w_k^2 \left\| D_k(x) - D_k(\bar{x}) \right\|^2 \,.
\end{align*}

We considered the Daubechies orthogonal family with $4$ vanishing moments and five levels of decomposition as a wavelet basis (see Appendix \ref{sec:complements_TS}), as implemented by \emph{PyWavelets} \citep{lee_pywavelets_2019}.

\subsection{Signature-consistency loss}
\label{sec:signature_loss}

Let $S^{(\leq k)}(x)$ be the signature transform of $x$ up to level $k$. This representation of a path is composed of a collection of iterated integrals on several levels, as defined in Section \ref{sec:background_ts}. We define the weighted signature-consistency loss as:

\[ \mathcal{L}_S = \sum_{k=1}^K w_k^2 \left\| S^{(\leq k)}(x) - S^{(\leq k)}(\bar{x}) \right\|^2 \,. \]

Notice that the injectivity of the full (untruncated) signature only holds when one of the axes is a monotone path shared between paths \citep{cass_lecture_2024}, and when all paths share a common departing point. Hence, a common practice is to augment the time series with the time coordinate and to append a starting point to every sequence. To apply this, we consider series to be of dimension $d+1$ by simply adding the timestamps $0 \leq t_i \leq 1,\ i = 1, \dots, m$ as the first coordinate, but only to one-dimensional series (for which the signature would not be defined otherwise). We do not append a starting point, meaning that $S^{(\leq k)}$ will be invariant to path translation \citep{kidger_signatory_2020}.

\section{Weighting scheme and theoretical justification of the consistency losses}
\label{sec:norm_equivalences}
\subsection{Weighting scheme definition}
The structure of the transforms presented in Section \ref{sec:methodology} allows us to focus the computational resources on features that are less emphasised by the Euclidean reconstruction error of vanilla autoencoders. To this end, we will consider weighting schemes to increase the importance of higher-level features, while allowing us to draw theoretical connections between our losses and existing norms in the literature.
 
Indeed, we define the weighting scheme for the Fourier loss as
\[ w_k=\left(1+|\xi_k|^2\right)^s \,. \]
This scheme will be referred to as the Sobolev weighting scheme, inspired by the connection with Sobolev norms, which we outline in Section~\ref{sec:sobolev}.
For the Signature and Wavelet transforms, where the levels consist of several components, the weighting scheme is set to: 
\[ w_k = \left(1+\left(\frac{k}{k_\text{max}}\right)^2\right)^s \,,\]
with $w_k$ denoting the level-indexed counterpart of the frequency bin-indexed $w_k$. This choice emulates the effect of the frequency weighting $w_k=(1+|\xi_k|^2)^s$ as it augments the weights as the level of detail increases, without fully neglecting low levels. Dividing by $k_\text{max}$, i.e., the maximum level considered in practice, ensures that the scale remains similar with different numbers of levels. The effect of varying the exponent $s$ is included in Appendix \ref{sec:sob_weighting}, with $s=1$ being used throughout the experiments.

In the rest of this section, we will establish the aforementioned theoretical connections for the Fourier and Signature consistency losses (Sections \ref{sec:sobolev} and \ref{sec:sig_kernel} respectively).

\subsection{Links to Sobolev norms}
\label{sec:sobolev}

Remarkably, our transform-consistency losses amount to constraining the functional class of the generated samples. Indeed, our alignment pushes the underlying path of samples to lie in the same Sobolev space as the data, hence forcing them to share the same degree of regularity. To formalise this, we present the discrete form of the Sobolev norm, which can be written in terms of the Fourier transform:
\[ \| x \|_{H^s}^2 := \sum_{k=1}^{M-1} (1 + |\xi_k|^2)^s \, \| \mathcal{F}[x](k) \|^2 , \]
where $\xi_k$ corresponds to the $k^\text{th}$ frequency. 
\begin{proposition}
\label{prop:sobolev_equivalence}
    Let $\mathcal{L}^2_{\mathcal{F},s}$ denote the transform-consistency loss in Equation \ref{eq:fourier_loss}, with disjoint frequency bins $B_k$. Denote $m_k:=\inf_{\xi\in B_k} (1+|\xi|^2)^{s/2}$ and $M_k:=\sup_{\xi\in B_k}(1+|\xi|^2)^{s/2}$. Then, for any weighting scheme $\{w_k\}_1^K$ such that $m_k \leq w_k \leq M_k$, $\| x-\bar{x} \|_{H^s}^2$ is equivalent to $\mathcal{L}^2_{\mathcal{F},s}(x,\bar{x})$.
\end{proposition}

The proof is provided in Appendix \ref{sec:proof_sobolev}. This holds in particular for separating the frequency bins independently, in which case $w_k:=(1+|\xi_k|^2)$ would yield exactly the Sobolev norm. Using the triangle inequality, it is straightforward to verify the following corollary:

\begin{corollary}
    Minimising $\mathcal{L}_\mathcal{F}(x,\bar{x})$ under the assumptions of Proposition \ref{prop:sobolev_equivalence} reduces the roughness increase that the reconstructions $\bar x$ show with respect to the input $x$, as measured by the $s$-Sobolev norm.
\end{corollary}

In $L^2$, for a real-valued function $f$ in $L^2$, the Sobolev norm can be written as $\| f \|_{H^s}^2 = \sum_{|\alpha| \leq s} \| D^\alpha f \|^2$, where $D^\alpha f$ are the partial derivatives of order $\alpha$ of $f$. In the context of machine learning, this norm has been used to reinforce training with derivative information of the target values with respect to the input \citep{czarnecki_sobolev_2017}. Here, a higher $s$ entails a stronger suppression of high frequencies in the context of signal processing.

\subsection{Wasserstein bound}
\label{sec:wasserstein}

The fine-tuning procedure makes the latent space reconstructions closer to the dataset according to norms prioritising higher frequencies. We will now see that this effectively reduces an upper bound of the distance between the data distribution and the distribution generated by latent flow matching. Indeed, let
\[ d_T(x_1, x_2) = \left( \sum_{k=1}^K \omega_k^2 \| T_k(x_1) - T_k(x_2) \| \right)^{1/2} \]
for a given representation $T = T_1, \dots, T_K$. Let us denote by $W^{(T)}$ the 2-Wasserstein distance that has $d_T$ as the cost function, that is,
\[ W_2^{(T)}(p_1, p_2) = \left( \min_{\pi \in \Pi(p_1, p_2)} \int_{\mathbb{R}^d \times \mathbb{R}^d} d_T(x_1, x_2)^2 \, d\pi(x_1, x_2) \right). \]

\begin{proposition}
\label{prop:wasserstein_bound}
Let $\tilde{p}_1$ be the distribution induced by latent flow matching, that is, $\tilde{p}_1 = \mathcal{D}_\# p_{\tilde{z}}$ with $p_{\tilde{z}}$ the distribution of latent samples produced by the trained flow in the latent space. Let $p_z = \mathcal{E}_\# p_\text{data}$ be the distribution of encoded data, and assume that both the decoder and the flow models are Lipschitz with constants $L_\mathcal{D}$ and $L_\theta$, respectively.

Then the 2-Wasserstein distance with a $T$-based metric between $p_\text{data}$ and $\tilde{p}_1$ can be bounded by:
\[ W_2^{(T)}(p_\text{data}, \tilde{p}_1) \leq \mathbb{E}_{p_\text{data}}\bigl[\mathcal{L}_T(x, \tilde{x})\bigr]^{1/2} + L_{\mathcal{D},T} \, e^{\frac{1+2L_\theta}{2}} H(u_\theta)^{1/2}, \]
where $L_{\mathcal{D},T} := \sup_{z \neq z'} \frac{d_T(\mathcal{D}(z), \mathcal{D}(z'))}{\| z - z' \|_2}$ is the Lipschitz constant of $d_T$, $\tilde{x} = \mathcal{D}(E(x))$ and $H(u_\theta)$ is the flow matching objective:
\[ H(u_\theta) = \int_0^1 \int_{\mathbb{R}^d} \| u_t(z_t) - u_\theta(z, t) \|^2 \, dq_t(z_t) \, dt \]
given the marginal distribution $q_t$ of latents $z_t$.
\end{proposition}

The proof can be found in Appendix \ref{sec:wasserstein_proof}.

\begin{remark}
    Our transform-consistency losses decrease the bound on the Wasserstein distance by directly minimising the term $\mathbb{E}_{p_\text{data}}[\mathcal{L}_T(x, \tilde{x})]$. Moreover, the loss is also implicitly acting on the second term by pulling reconstructions and data closer with respect to $d_T$, which arguably decreases the $d_T$-Lipschitz constant $L_{\mathcal{D},T}$ in the support of $p_z$. Consequently, the full $T$-based Wasserstein distance is reduced by the proposed alignment.
\end{remark}

\subsection{The signature kernel}
\label{sec:sig_kernel}
We next show that the Signature-consistency loss approximates the action of a kernel, that is, a similarity function for a pair of data points. Let us consider the signature kernel given by \citep{cass_lecture_2024}
\begin{equation}
\label{eq:sig_kernel}
    k_w([x],[y]) = \langle S(x), S(y) \rangle_w,
\end{equation}
and refer to Appendix \ref{sec:unparameterised_paths} for further details of this kernel. A notable property is that $k_w$ allows us to define a reproducible kernel Hilbert space (RKHS, \cite{aronszajn_theory_1950}), which we briefly introduce in Appendix \ref{sec:RKHS}. We formalise this connection in Proposition \ref{prop:sigkernel} below.

\begin{proposition}
\label{prop:sigkernel}
Let $s \in \mathbb{R}_+$ and set $ w_k = (1+(\frac{k}{k_\text{max}})^2)^s $ for $k_\text{max}\in\mathbb{N}_+$. Then, the $w$-signature kernel given by
\[ k_w(x,y) = \sum_{k=0}^{\infty} w_k\, \langle S(x)^{(k)}, S(y)^{(k)} \rangle \]
defines a unique RKHS $\mathcal{H}_w$ satisfying the reproducing property via $k_w$. In particular, minimising the infinite transform-consistency loss with weighting function $w_k$ defined as
\[ \mathcal{L}_S^\infty = \sum_{k=1}^{\infty} w_k^2 \, \| S^{(k)}(x) - S^{(k)}(\bar{x}) \|^2 \]
amounts to minimising $\|x - \bar{x}\|_{\mathcal{H}_w}^2$, where $\|x\|_{\mathcal{H}_w} = \sqrt{k_w(x,x)}$.
\end{proposition}

The proof of Proposition \ref{prop:sigkernel} is provided in Appendix \ref{sec:sigkernel_proof}. The kernel associated with the signature loss has the property of uniquely characterising the underlying probability measure of paths. This characterisation holds via the kernel mean embedding $M(\mu) := \mathbb{E}_{x\sim\mu}[k_\theta(x,\cdot)]$, for $\mu$ in the Borel set $\mathcal{P}(K)$ of probability measures on a compact set $K$. In the setup of Proposition~\ref{prop:sigkernel}, $M$ is injective on $\mathcal{P}(K)$, hence distinguishing between measures on $\mathcal{P}(K)$. We refer to \cite{cass_lecture_2024} and \cite{simon-gabriel_kernel_2018} for a formal statement and proof. On the other hand, Lemma 2.3.1 from \cite{cass_lecture_2024} ensures that the truncated signature transform-consistency loss that we use in practice converges to $\mathcal{L}_S^\infty$ as the number of levels $K$ grows to infinity.

\section{Experiments}
\label{sec:experiments}

\subsection{Quantitative evaluation}
\paragraph{Architecture.} We begin by training a base model (unaligned) as specified in Section \ref{sec:base_model}. The AE consists of an encoder-decoder pair, where each consists of two convolutional layers with a stride equal to $2$, after which a linear layer projects the output to the final latent representation. The flow model regressing the velocity vector field is composed of a multilayer perceptron-based U-NET \citep{ronneberger_u-net_2015} with depth $2$. Specific dimensionalities of the network are specified in Table~\ref{tab:model_details}.

\paragraph{AE alignment.} We then align the encoder-decoder pair by fine-tuning the base model for $25$ epochs. This separate optimisation stage corresponds to minimising the transform-consistency losses as an additional term to the reconstruction and adversarial losses introduced in Section \ref{sec:latent_spaces}. This joint optimisation yielded more stable training than only minimising the transform-consistency loss. The weighting schemes are set to Sobolev for all three variants, as specified in Section \ref{sec:sobolev}, with $s=1$. Other details, such as training times and dimensionality of the transforms, are included in Appendix \ref{sec:exp_details}.

\paragraph{Flow model training and sampling.} The flow model was trained for $500$ epochs for multivariate series and $250$ for single-channel ones, after which the models showed no improvement. A new flow is trained for each new aligned latent space, although the fact that the geometry of the latent space does not change drastically suggests that fine-tuning an existing flow is probably a positive prospect. We emphasise that our alignment framework is intended to be carried out before training the flow; the fine-tuning of the autoencoder is the only extra computational overhead. Sampling was carried out using a $10$-step Euler method \citep{hairer_solving_1993}.

\paragraph{Datasets.} This study focuses on long-range signals. To this end, we considered the datasets Weather \citep{kolle_documentation_2024}, Exchange rates \citep{lai_modeling_2018} and HEPC \citep{uci_machine_learning_repository_household_2024}. These datasets are comprised of multiple channels (1, 8 and 14 for HEPC, exchange rates and weather, respectively), and they reflect different real-world scenarios. Consequently, their frequency structure is different, which highlights the advantages of a methodology that aligns representations to the underlying dataset.  

\paragraph{Metrics and baselines.} We followed the evaluation protocol used in \cite{yoon_time-series_2019} and \cite{barancikova_sigdiffusions_2024}, adapted for unconditional time series generation. The discriminative score reflects how well synthetic data can fool a classifier, measured via the out-of-sample accuracy of a recurrent neural network (RNN). The predictive score is the loss of an RNN trained on real data to forecast the next point, evaluated on generated samples.
We also applied a Kolmogorov–Smirnov (KS) test to compare the empirical distributions of real test series with those of generated data. We report the KS statistic and the percentage of points where the null hypothesis (equal distributions) is rejected at a 5\% significance level, applied to the distribution of values of the time series evaluated at a given time stamp ($t=300,500,700$ and $900$). For readability, the discriminative score is shifted by subtracting 0.5, so all metrics improve as they move closer to 0. Our models are compared with four diffusion-based baselines: SigDiffusion \citep{barancikova_sigdiffusions_2024}, diffusion-TS \citep{yuan_diffusion-ts_2023}, CSPD-GP \citep{bilos_modeling_2023}, and DDO \citep{lim_regular_2023}. 

\paragraph{Results.}
Table \ref{tab:unconditional_eval} reports the performance of our proposed models and baselines. The base model and the transform-aligned variants show clear gains over these state-of-the-art methods, achieving stronger discriminative and predictive results. The regularised versions usually improve on the base model, especially on the discriminative score. We include a discussion of why this is the case in Section \ref{sec:high_freq_role}. Furthermore, our proposed method generates $1000$ samples in under one second on an RTX3090 GPU, which represents a significant improvement over the considered benchmarks: between 22 and 97 seconds for SigDiffusion and more than one minute for the rest using the same hardware. Training times are also included in Appendix \ref{sec:exp_details}. Fine-tuning with the signature-consistency loss incurs longer training times, which is discussed in Appendix \ref{sec:signature_limitations}.

\begin{table*}[!ht]
\caption{Discriminative, predictive and KS test metrics (KS score and rejection percentage in brackets). Less is better for all metrics. The results are computed for 1000 synthetically generated time series of length 1000, compared to 1000 unseen series from the original datasets. The KS tests are computed over the temporal indices $t=300,500,700,900$.}
\centering
\scriptsize{
\label{tab:unconditional_eval}
\hspace*{-1.5cm}
\begin{tabular}{@{}llllllll@{}}
\toprule
dataset & Model & \begin{tabular}[c]{@{}l@{}}Discriminative\\ score\end{tabular} & \begin{tabular}[c]{@{}l@{}}Predictive\\ score\end{tabular} & \begin{tabular}[c]{@{}l@{}}KS\\ t=300\end{tabular} & \begin{tabular}[c]{@{}l@{}}KS\\ t=500\end{tabular} & \begin{tabular}[c]{@{}l@{}}KS\\ t=700\end{tabular} & \begin{tabular}[c]{@{}l@{}}KS\\ t=900\end{tabular} \\ \midrule
 & Base + Fourier cons. (ours) & \textbf{0.018±.012} & 0.043±.001 & \textbf{0.15 (3\%)} & 0.15 (4\%) & \textbf{0.14 (2\%)} & 0.15 (3\%) \\
 & Base + wavelet cons. (ours) & 0.022±.015 & 0.043±.000 & \textbf{0.15 (3\%)} & 0.17 (12\%) & 0.15 (4\%) & \textbf{0.15 (2\%)} \\
 & Base + signature cons. (ours) & 0.208±.100 & \textbf{0.042±.000} & 0.17 (12\%) & \textbf{0.15 (3\%)} & 0.16 (5\%) & 0.15 (4\%) \\ \cmidrule(l){2-8} 
\multicolumn{1}{c}{HEPC} & Base model (ours) & 0.211±.103 & 0.043±.000 & 0.16 (7\%) & 0.15 (5\%) & \textbf{0.14 (2\%)} & 0.16 (7\%) \\ \cmidrule(l){2-8} 
 & SigDiffusion & 0.070±.032 & 0.050±.012 & 0.20 (16\%) & 0.18 (9\%) & 0.19 (12\%) & 0.21 (22\%) \\
 & DDO ($\gamma$ = 1) & 0.081±.019 & 0.044±.001 & 0.25 (46\%) & 0.23 (38\%) & 0.25 (46\%) & 0.25 (51\%) \\
 & Diffusion-TS & 0.438±.057 & 0.066±.022 & 0.85 (100\%) & 0.87 (100\%) & 0.82 (100\%) & 0.88 (100\%) \\
 & CSPD-GP (RNN) & 0.415±.045 & 0.108±.002 & 0.52 (100\%) & 0.53 (100\%) & 0.55 (100\%) & 0.56 (100\%) \\
 & CSPD-GP (Transformer) & 0.500±.000 & 0.551±.028 & 1.00 (100\%) & 1.00 (100\%) & 1.00 (100\%) & 1.00 (100\%) \\ \midrule
 & Base + Fourier cons. (ours) & 0.072±.015 & 0.030±.001 & 0.18 (18\%) & \textbf{0.18 (17\%)} & 0.19 (18\%) & 0.19 (19\%) \\
 & Base + wavelet cons. (ours) & \textbf{0.069±.023} & \textbf{0.029±.001} & \textbf{0.18 (17\%)} & 0.18 (18\%) & \textbf{0.18 (15\%)} & 0.19 (19\%) \\
 & Base + signature cons. (ours) & 0.114±.026 & \textbf{0.029±.001} & 0.19 (21\%) & 0.19 (18\%) & 0.19 (18\%) & 0.19 (21\%) \\ \cmidrule(l){2-8} 
\multicolumn{1}{c}{Exchange} & Base model (ours) & 0.074±.026 & \textbf{0.029±.001} & 0.19 (18\%) & 0.18 (18\%) & 0.18 (16\%) & \textbf{0.18 (17\%)} \\ \cmidrule(l){2-8} 
\multicolumn{1}{c}{} & SigDiffusion & 0.278±.062 & 0.057±.001 & 0.31 (80\%) & 0.28 (67\%) & 0.28 (65\%) & 0.31 (74\%) \\
 & DDO ($\gamma=1$) & 0.326±.102 & 0.094±.004 & 0.24 (41\%) & 0.24 (42\%) & 0.25 (45\%) & 0.25 (45\%) \\
 & Diffusion-TS & 0.401±.196 & 0.120±.016 & 0.72 (100\%) & 0.71 (100\%) & 0.70 (100\%) & 0.69 (100\%) \\
 & CSPD-GP (RNN) & 0.500±.001 & 0.273±.100 & 0.59 (100\%) & 0.56 (100\%) & 0.55 (100\%) & 0.56 (100\%) \\
 & CSPD-GP (Transformer) & 0.500±.000 & 0.432±.074 & 1.00 (100\%) & 0.99 (100\%) & 0.98 (100\%) & 0.99 (100\%) \\ \midrule
\multicolumn{1}{c}{} & Base + Fourier cons. (ours) & 0.350±.153 & \textbf{0.158±.002} & 0.22 (20\%) & 0.22 (16\%) & \textbf{0.21 (16\%)} & \textbf{0.20 (15\%)} \\
 & Base + wavelet cons. (ours) & 0.282±.097 & 0.163±.005 & \textbf{0.22 (18\%)} & \textbf{0.21 (14\%)} & 0.22 (18\%) & 0.21 (20\%) \\
 & Base + signature cons. (ours) & \textbf{0.227±.190} & 0.158±.003 & 0.23 (22\%) & 0.21 (16\%) & 0.23 (17\%) & 0.21 (18\%) \\ \cmidrule(l){2-8} 
\multicolumn{1}{c}{Weather} & Base model (ours) & 0.379±.063 & 0.176±.015 & 0.23 (24\%) & 0.23 (18\%) & 0.24 (22\%) & 0.23 (20\%) \\ \cmidrule(l){2-8} 
 & SigDiffusion & 0.350±.080 & 0.168±.001 & 0.35 (82\%) & 0.34 (80\%) & 0.33 (76\%) & 0.34 (78\%) \\
 & DDO ($\gamma$ = 1) & 0.356±.196 & 0.307±.007 & 0.26 (45\%) & 0.27 (52\%) & 0.26 (46\%) & 0.26 (47\%) \\
 & Diffusion-TS & 0.498±.003 & 0.438±.035 & 0.49 (100\%) & 0.50 (100\%) & 0.50 (100\%) & 0.49 (100\%) \\
 & CSPD-GP (RNN) & 0.500±.000 & 0.505±.007 & 0.57 (100\%) & 0.56 (100\%) & 0.56 (100\%) & 0.56 (100\%) \\
 & CSPD-GP (Transformer) & 0.500±.000 & 0.490±.000 & 0.91 (100\%) & 0.92 (100\%) & 0.91 (100\%) & 0.91 (100\%) \\ \bottomrule
\end{tabular}}
\hspace*{-1.5cm}
\end{table*}

\subsection{Qualitative evaluation}
\paragraph{Waveform visualisation.}
Figure \ref{fig:flow_weather} shows real and generated samples for the Weather dataset \citep{kolle_documentation_2024} using both base and aligned models (via the Fourier transform). The aligned version shows a significant improvement over the base model in terms of smoothness. This is further noticeable when observing the reconstructions, shown in Figure \ref{fig:flow_weather_reconstructions}, where it is evident that the compression induces a mismatch in local structure.
\begin{figure}[]
    \centering
    \includegraphics[width=0.9\linewidth]{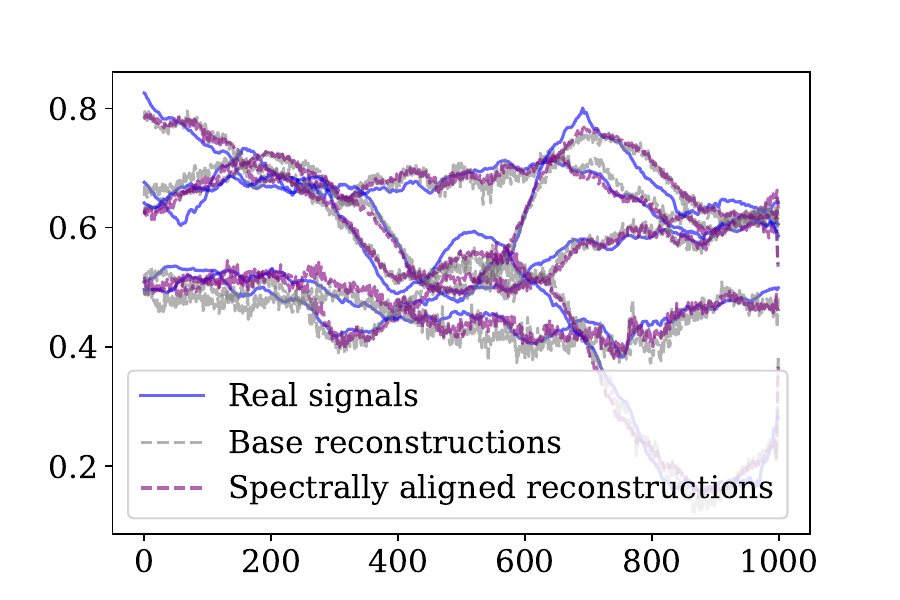}
    \caption{Signals from the weather dataset \citep{kolle_documentation_2024} and their reconstructions from the base and Fourier-aligned autoencoders.}
    \label{fig:flow_weather_reconstructions}
\end{figure}

\begin{figure}[]  
\centering
\begin{subfigure}{0.45\textwidth}
    \centering
    \includegraphics[width=\linewidth]{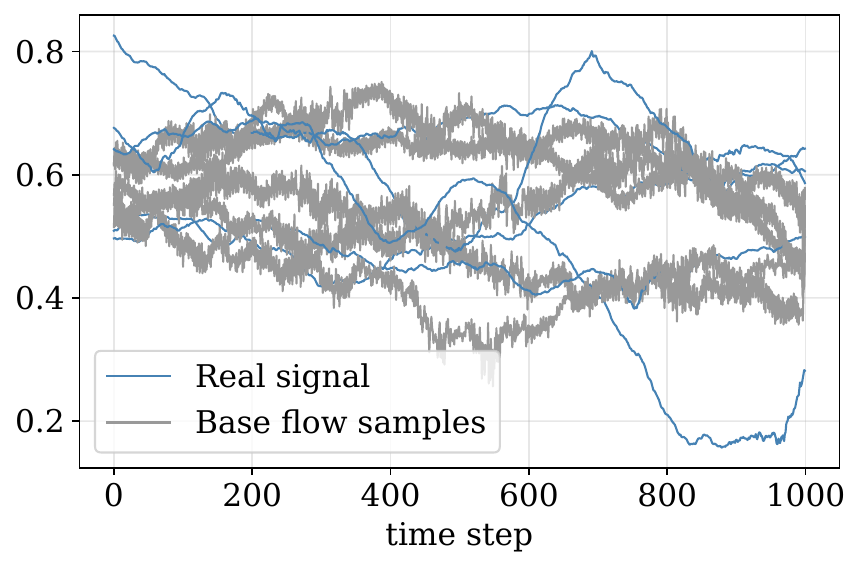}
    \caption{Base model}
\end{subfigure}
\begin{subfigure}{0.45\textwidth}
    \centering
    \includegraphics[width=\linewidth]{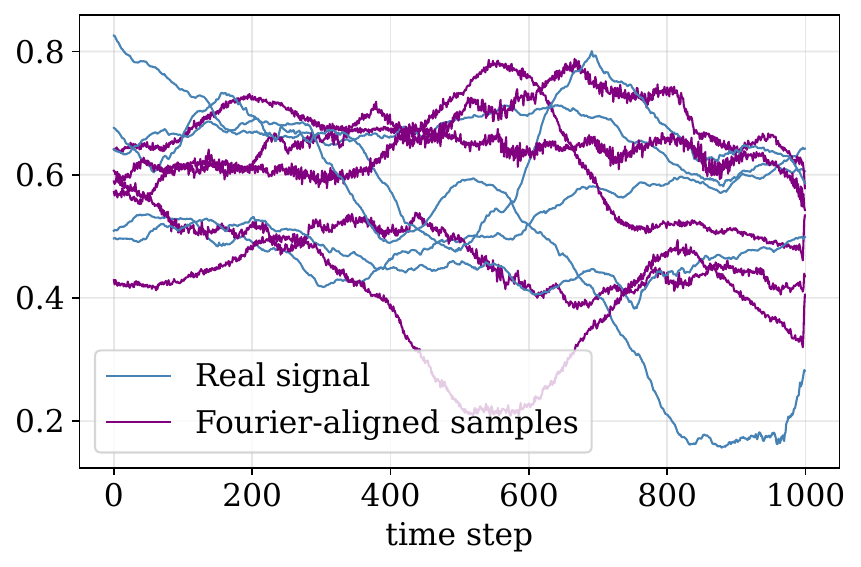}
    \caption{Fourier-aligned model}
\end{subfigure}
\caption{First dimension of FM-generated samples vs testing data for the weather dataset \citep{kolle_documentation_2024}. The subcaption shows the underlying models used. The five samples from both models were drawn by solving the FM ODE from the same Gaussian source sample.}
\label{fig:flow_weather}
\end{figure}

\paragraph{Density and t-SNE visualisations.}
We include visualisations of the generated samples compared to SigDiffusions \cite{barancikova_sigdiffusions_2024} in Figures \ref{fig:hepc_comparison}, \ref{fig:exchange_comparison} and \ref{fig:weather_comparison} for the HEPC, exchange rates and weather datasets, respectively. These consist of probability density estimates and t-SNE dimensionality reduction.  Our samples show a better match to the structure of the underlying datasets. We have shown the Fourier-aligned variants for the plots.

\begin{figure*}[T!]
\centering
\begin{subfigure}{0.24\textwidth}
    \centering
    \includegraphics[width=\linewidth]{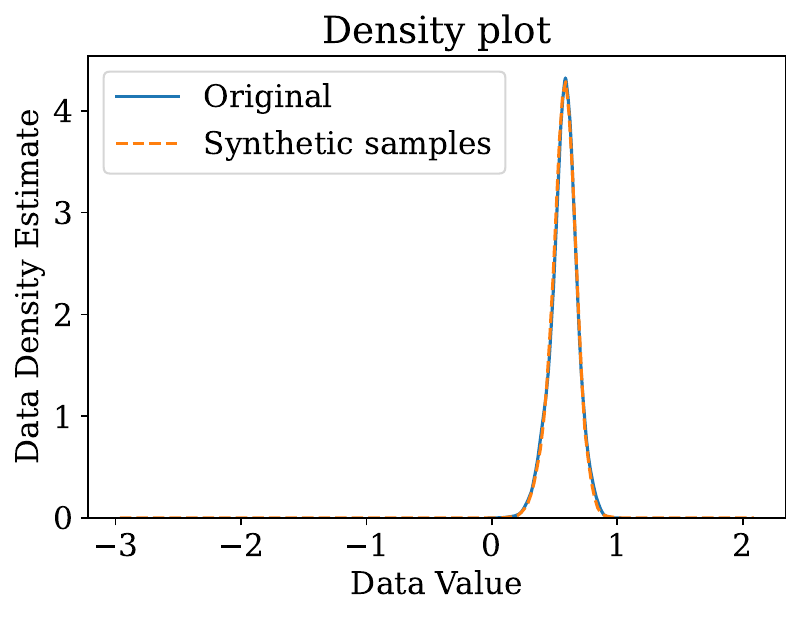}
    \caption{Ours (density)}
\end{subfigure}
\begin{subfigure}{0.24\textwidth}
    \centering
    \includegraphics[width=\linewidth]{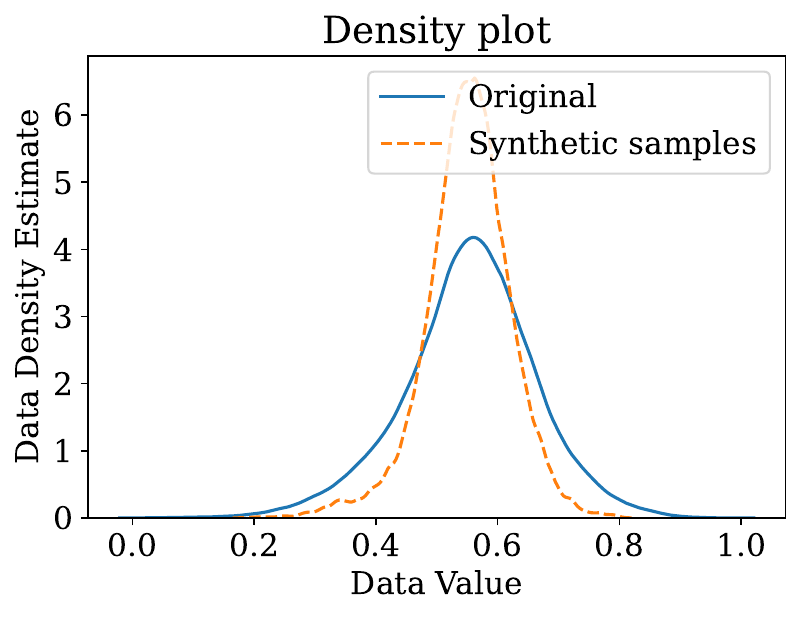}
    \caption{SigDiffusions (density)}
\end{subfigure}
\begin{subfigure}{0.24\textwidth}
    \centering
    \includegraphics[width=\linewidth]{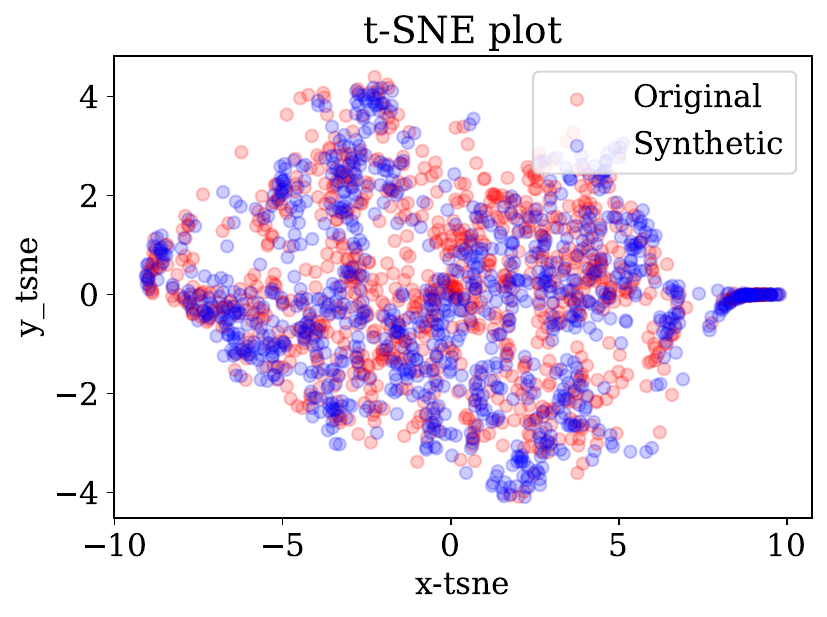}
    \caption{Ours (t-SNE)}
\end{subfigure}
\begin{subfigure}{0.24\textwidth}
    \centering
    \includegraphics[width=\linewidth]{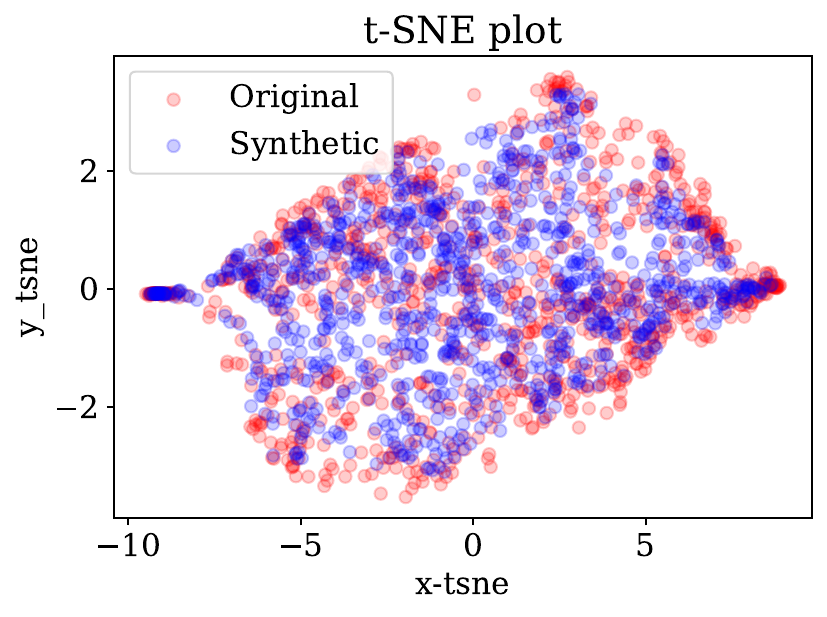}
    \caption{SigDiffusions (t-SNE)}
\end{subfigure}
\caption{Comparison of density and t-SNE plots on the HEPC dataset.}
\label{fig:hepc_comparison}
\end{figure*}

\begin{figure*}[T!]
\centering
\begin{subfigure}{0.24\textwidth}
    \centering
    \includegraphics[width=\linewidth]{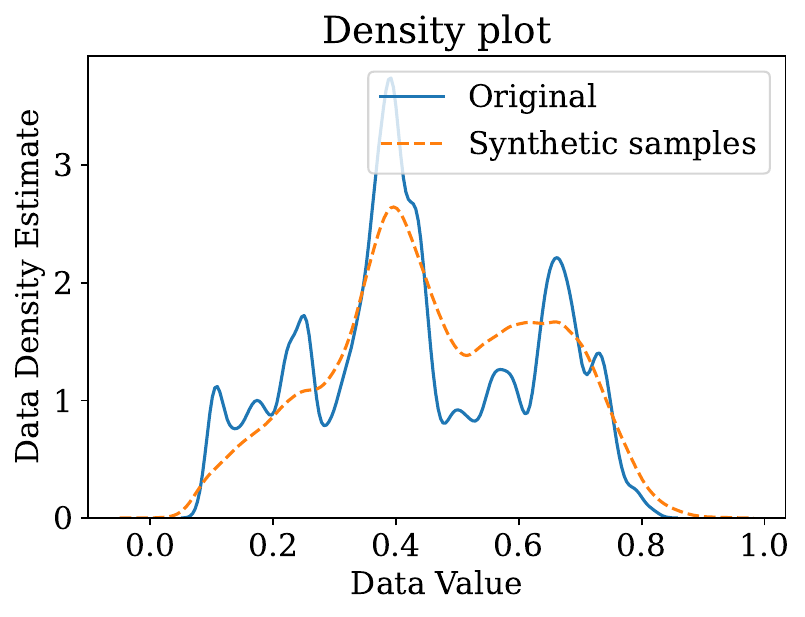}
    \caption{Ours (density)}
\end{subfigure}
\begin{subfigure}{0.24\textwidth}
    \centering
    \includegraphics[width=\linewidth]{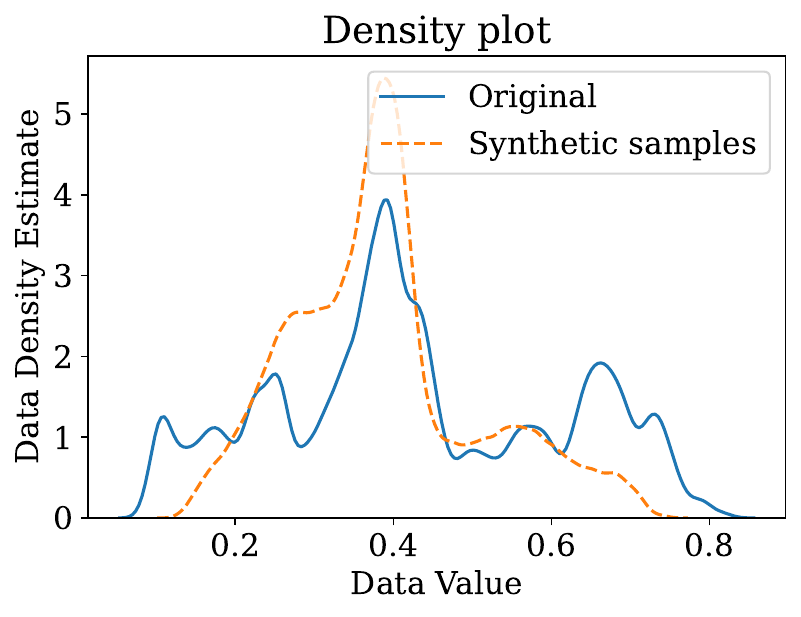}
    \caption{SigDiffusions (density)}
\end{subfigure}
\begin{subfigure}{0.24\textwidth}
    \centering
    \includegraphics[width=\linewidth]{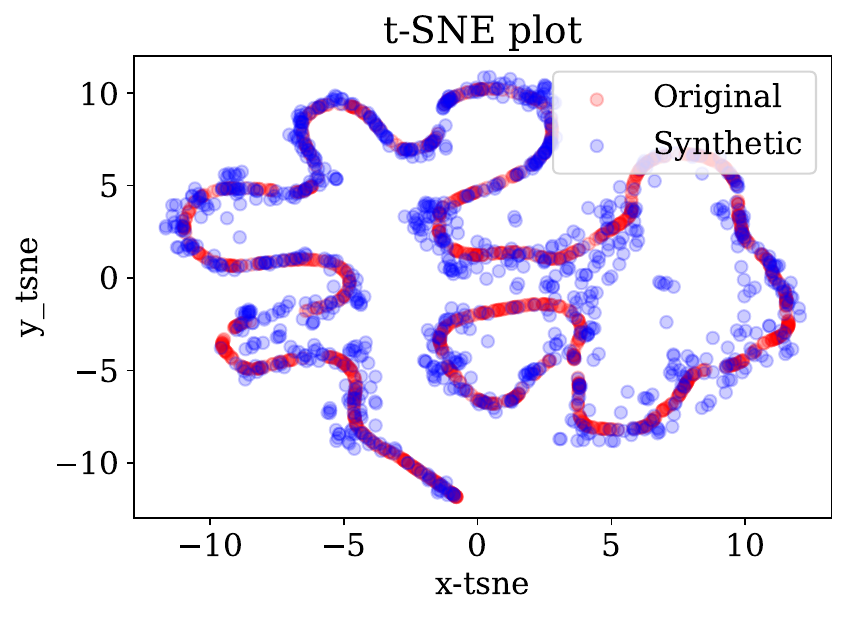}
    \caption{Ours (t-SNE)}
\end{subfigure}
\begin{subfigure}{0.24\textwidth}
    \centering
    \includegraphics[width=\linewidth]{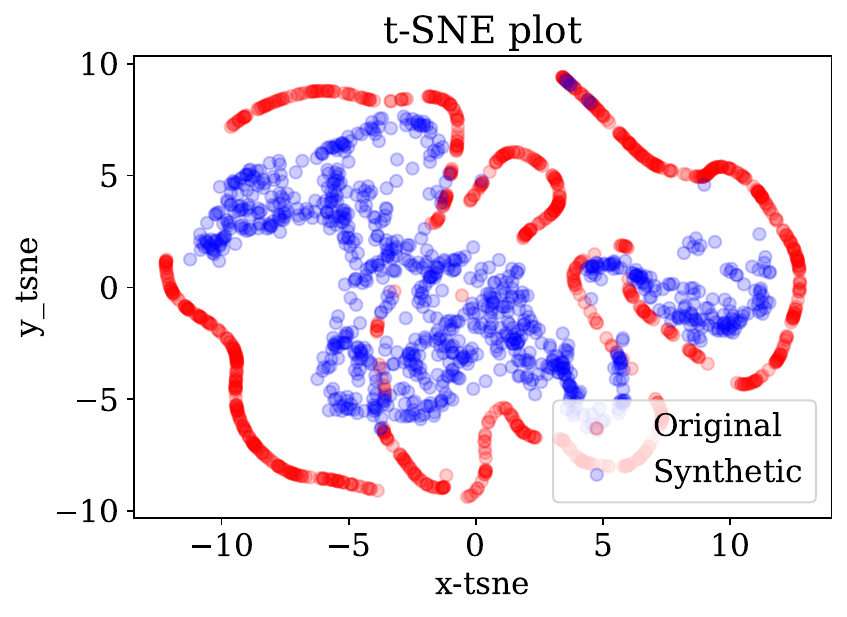}
    \caption{SigDiffusions (t-SNE)}
\end{subfigure}
\caption{Comparison of density and t-SNE plots on the Exchange dataset.}
\label{fig:exchange_comparison}
\end{figure*}

\begin{figure*}[T!]
\centering
\begin{subfigure}{0.24\textwidth}
    \centering
    \includegraphics[width=\linewidth]{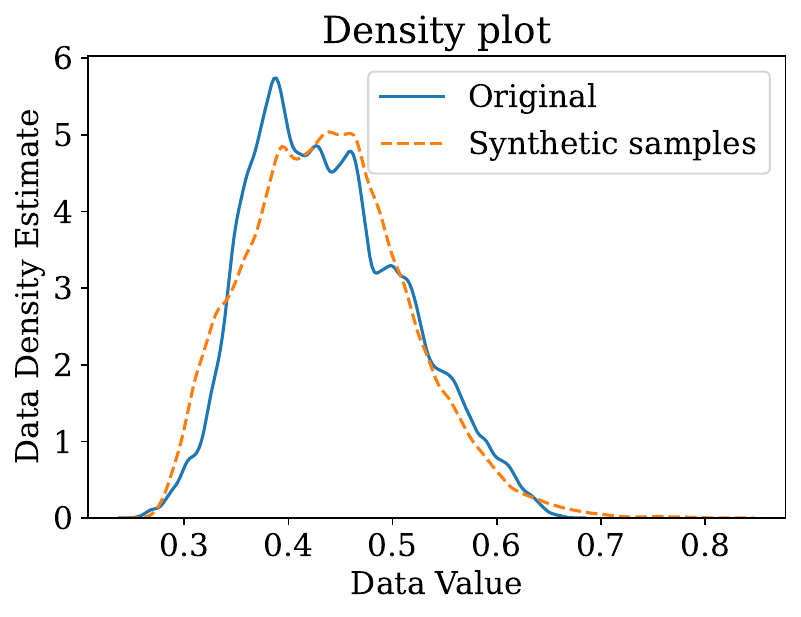}
    \caption{Ours (density)}
\end{subfigure}
\begin{subfigure}{0.24\textwidth}
    \centering
    \includegraphics[width=\linewidth]{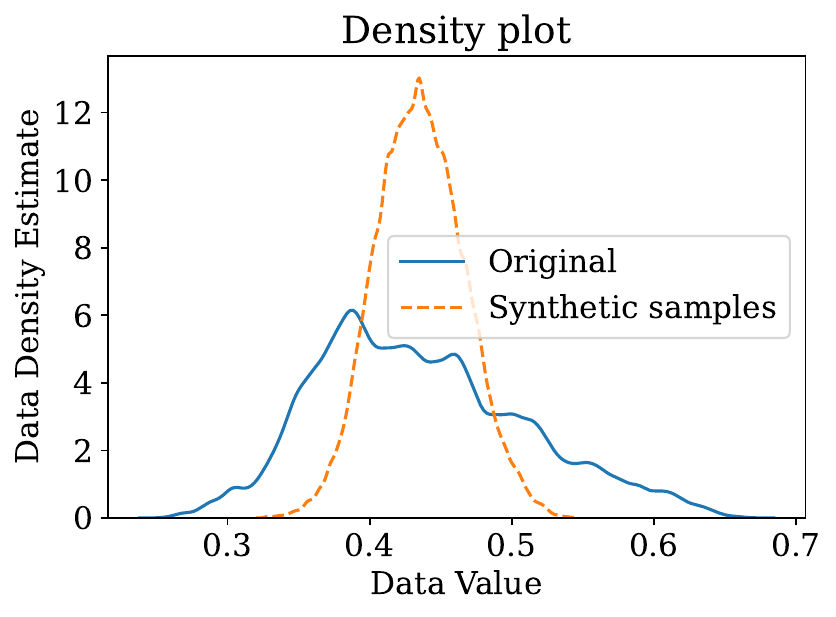}
    \caption{SigDiffusions (density)}
\end{subfigure}
\begin{subfigure}{0.24\textwidth}
    \centering
    \includegraphics[width=\linewidth]{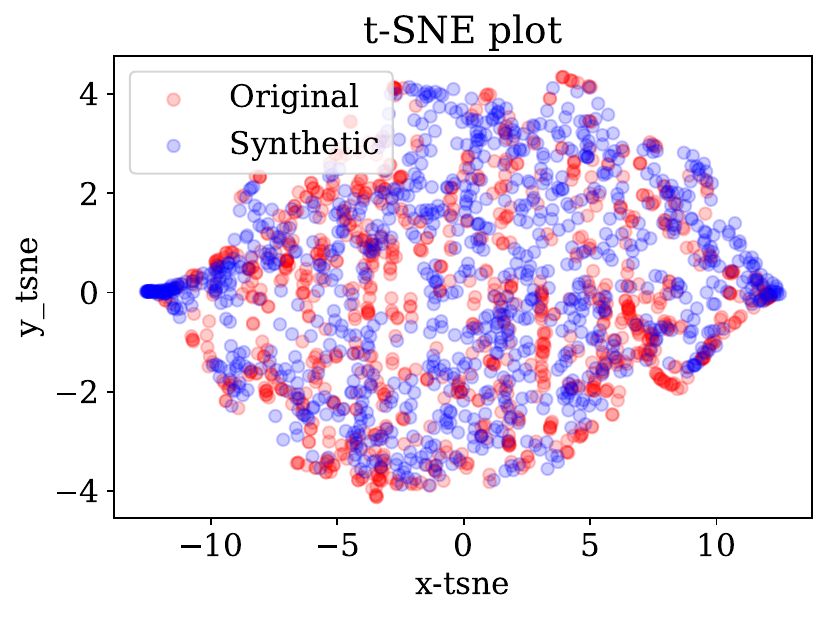}
    \caption{Ours (t-SNE)}
\end{subfigure}
\begin{subfigure}{0.24\textwidth}
    \centering
    \includegraphics[width=\linewidth]{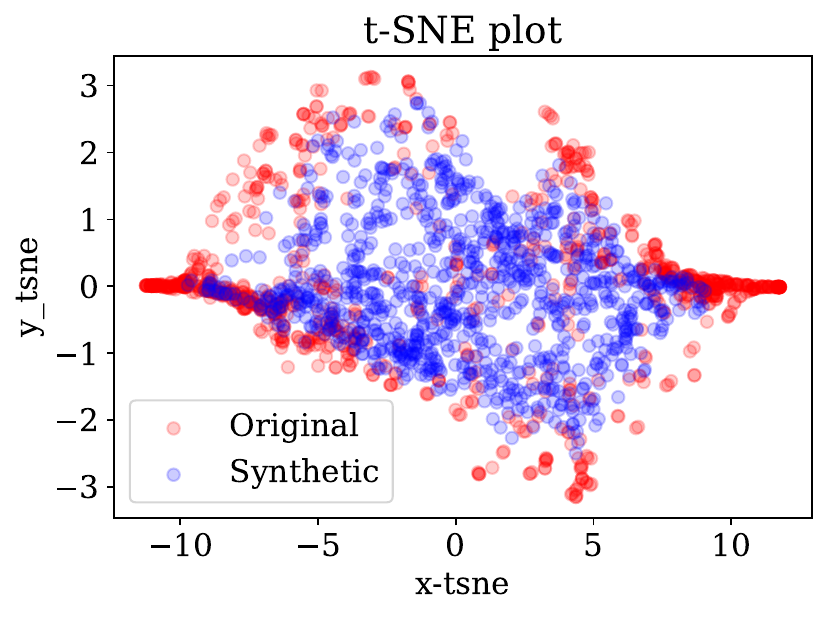}
    \caption{SigDiffusions (t-SNE)}
\end{subfigure}
\caption{Comparison of density and t-SNE plots on the Weather dataset.}
\label{fig:weather_comparison}
\end{figure*}

\subsection{Ablation: The role of high frequencies in improving generation}
\label{sec:high_freq_role}

In this subsection, we ablate the effect of assigning stronger weights to higher-order features of the considered TS transforms. For this purpose, we define alternative weighting schemes for the Fourier and wavelet consistency losses of Section \ref{sec:methodology}. Recall first that we consider decomposable transforms in $K$ levels. Consider the approximation-detail structure of the Wavelet transform (Section \ref{sec:wavelet_loss}). Indeed, a common approach to data compression would be to remove the detail coefficients corresponding to higher levels of detail (removing half of the wavelet coefficients). Akin to this compression, we define a truncated version of the transform-consistency loss by simply removing the lower level of details (i.e., setting their weight to zero).

A version of this is implemented for the Fourier representation by partitioning the frequency bins into ``dyadic levels'', i.e., such that each level of detail (frequencies with lower absolute value) corresponds approximately to $1/2^{k}$ of the frequency grid. The result is a structure that resembles the wavelet transform, to which we can apply the truncation by removing the top high-frequency bins. We show the results in Table \ref{tab:truncated} for the discriminative and predictive scores. We can observe, for both variants, that removing the last level of detail/frequency is detrimental to sample quality, particularly for the Fourier transform. This shows the relevance of high frequencies for improving samples, as these generally control the functional class to which the time series belongs. Moreover, we observe the benefits of stronger weights in higher frequencies or detail levels by comparing the Sobolev weighting to a vanilla Euclidean fine-tuning of the encoder-decoder pair (referred to as the uniform version). The result of not emphasising any level of detail lies between the truncated and the Sobolev variants in terms of the discriminative score. Hence, we can conclude that emphasising these over coarse levels is beneficial for new samples to resemble the data points, thus validating our choice of weighting scheme.

\begin{table}[ht]
\caption{Performance of the truncated version of the Fourier and wavelet consistency losses,
 using the Weather dataset. The abbreviations are disc. for the discriminative score, and pred. for the predictive score, respectively.}
\label{tab:truncated}
\centering
\begin{tabular}{@{}cccc@{}}
\toprule
transform & weights & disc. score $\downarrow$ & \multicolumn{1}{l}{pred. score $\downarrow$} \\ \midrule
wavelet & Sobolev & 0.28±.10 & 0.16±.00 \\
 & truncated & 0.47±.03 & 0.16±.00 \\ \midrule
Fourier & Sobolev & 0.35±.15 & 0.16±.00 \\
 & truncated & 0.5±.00 & 0.47±.01 \\ \midrule
None & uniform & 0.37±.11 & 0.16±.00 \\ \bottomrule
\end{tabular}
\end{table}

\subsection{The location of the reconstruction error}
\label{sec:error_location}

To reflect on how the models are benefiting from the fine-tuning stage, we analyse what proportion of the reconstruction error is reflected in each detail/decomposition level. By analysing the percentage of the reconstruction error in each frequency bin (see Figure~\ref{fig:freq_error}), we can see that for the HEPC and Weather datasets, the share of the reconstruction loss in high frequencies is disproportionate, even though they carry limited energy. We claim this is a consequence of neural networks' tendency to fit lower frequencies first, known as spectral bias \cite{rahaman_spectral_2019}. The transform-consistency losses with Sobolev weighting emphasise gradients towards components that contribute to perceptual fidelity, but are less prioritised in the standard objective. Overall, some of the spectral artefacts are attenuated by our alignment procedure. For instance, the aligned reconstructions show a less pronounced peak at the higher frequencies, for all three variants. Moreover, some peaks at intermediate frequencies are diminished, and some improvement is even observed at low frequencies.

\begin{figure}[]
\centering
\begin{subfigure}{0.5\textwidth}
    \centering
    \includegraphics[width=\linewidth]{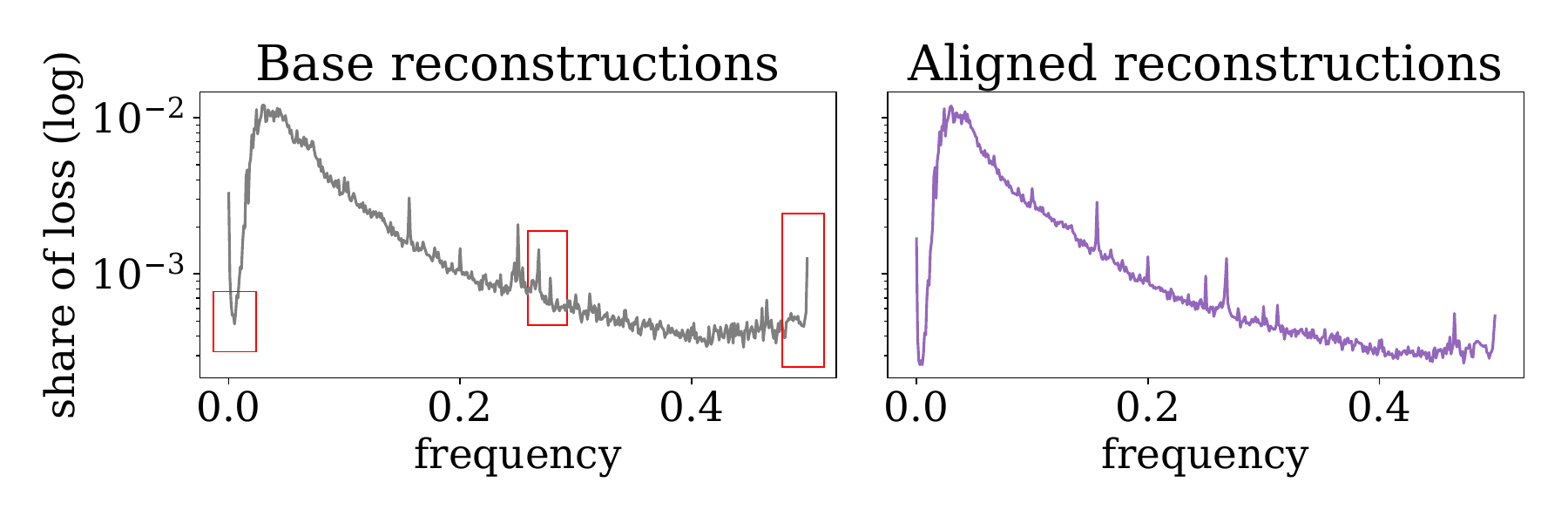}
    \caption{HPEC}
\end{subfigure}
\begin{subfigure}{0.5\textwidth}
    \centering
    \includegraphics[width=\linewidth]{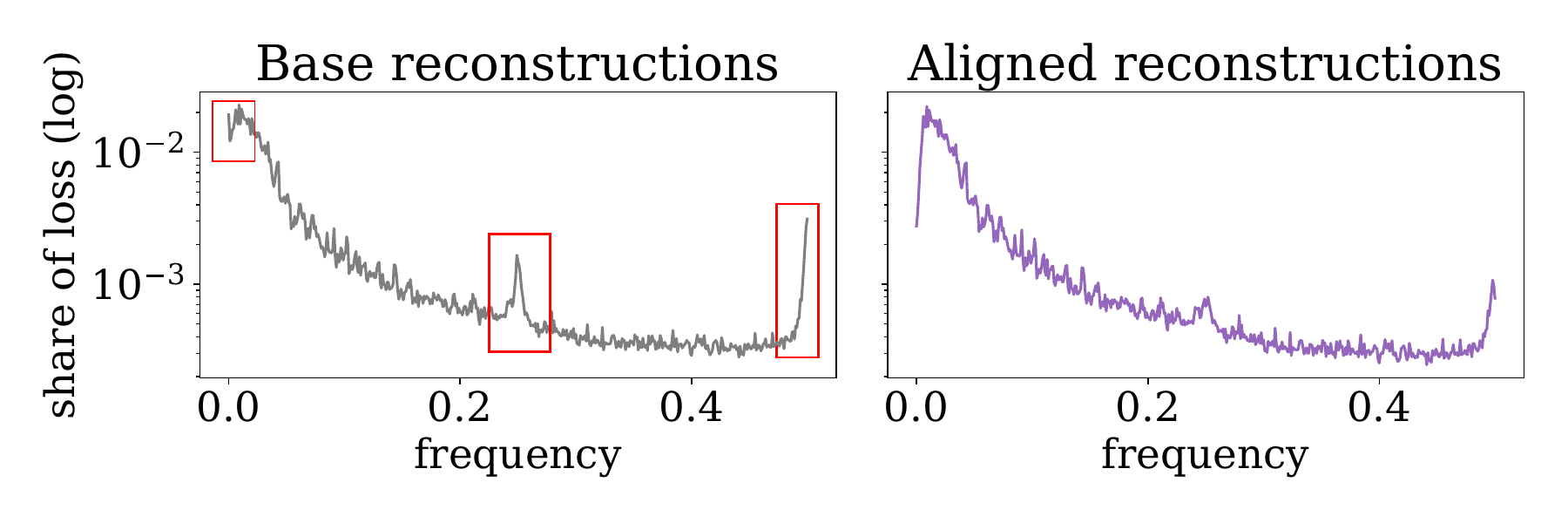}
    \caption{Exchange}
\end{subfigure}
\begin{subfigure}{0.5\textwidth}
    \centering
    \includegraphics[width=\linewidth]{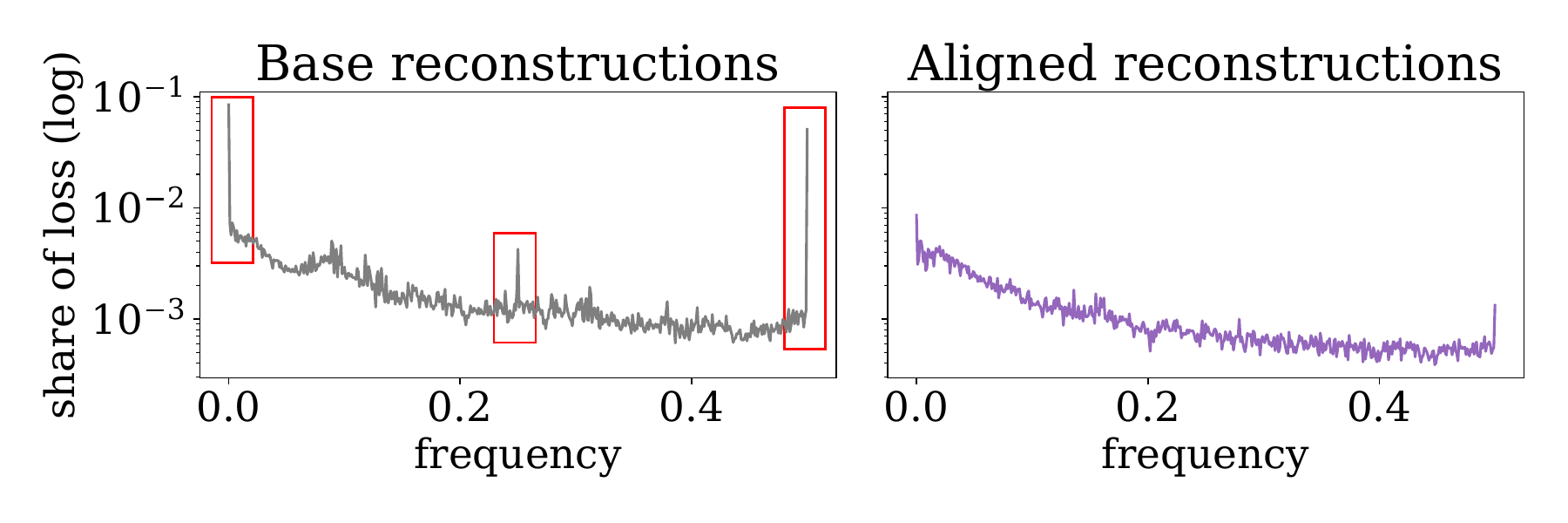}
    \caption{Weather}
\end{subfigure}
\caption{Sum of the reconstruction error for each frequency bin. The subcaption indicates the underlying dataset. The sections where the alignment helps decrease the reconstruction error are highlighted in red.}
\label{fig:freq_error}
\end{figure}

The effect of higher levels of detail having a high share of the reconstruction error is arguably even higher for the HEPC and weather datasets, when using the signature and wavelet transforms as underlying representations. This indicates the representation invariance of the spectral bias phenomenon. Figure~\ref{fig:wavelet_error} shows the per-level share of the reconstruction error for the wavelet transform, while Figure~\ref{fig:signature_error} shows it for the Signature transform. In both cases, we include the error shares for the aligned models relative to the original error.

\begin{figure}[]
\centering
\begin{subfigure}{0.35\textwidth}
    \centering
    \includegraphics[width=\linewidth]{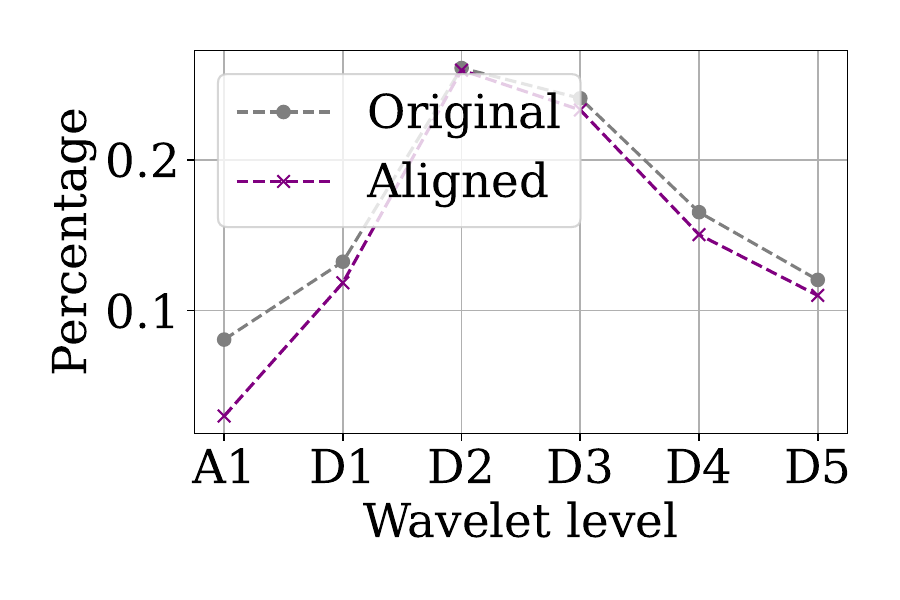}
    \caption{HPEC}
\end{subfigure}
\begin{subfigure}{0.35\textwidth}
    \centering
    \includegraphics[width=\linewidth]{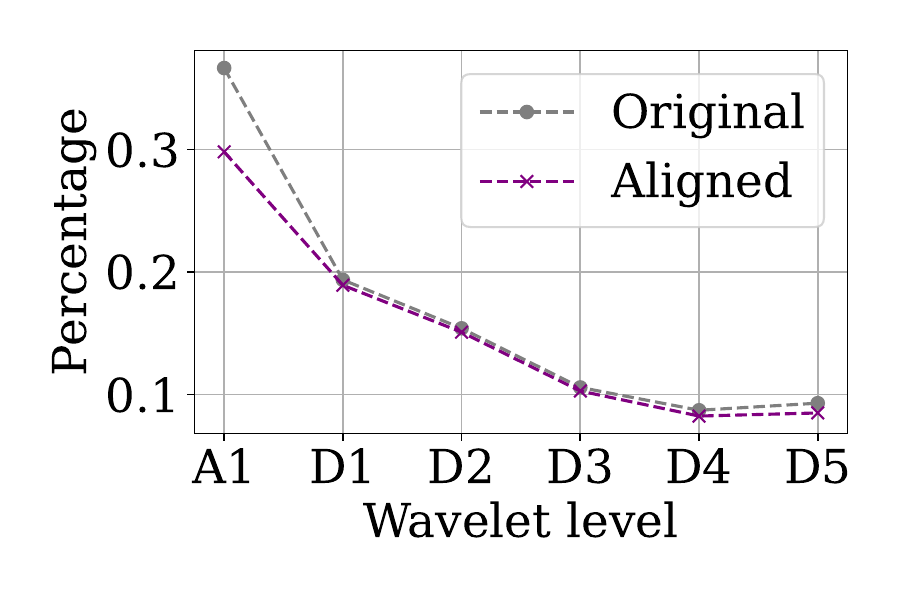}
    \caption{Exchange}
\end{subfigure}
\begin{subfigure}{0.35\textwidth}
    \centering
    \includegraphics[width=\linewidth]{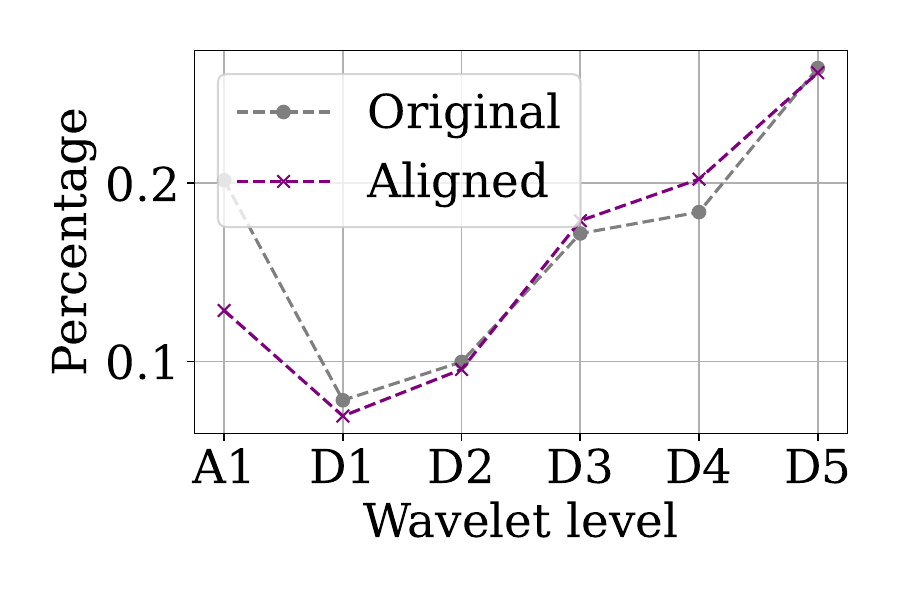}
    \caption{Weather}
\end{subfigure}
\caption{Sum of the reconstruction error averaged over each Wavelet level. The subcaption indicates the underlying dataset.}
\label{fig:wavelet_error}
\end{figure}

\begin{figure}[]
\centering
\begin{subfigure}{0.35\textwidth}
    \centering
    \includegraphics[width=\linewidth]{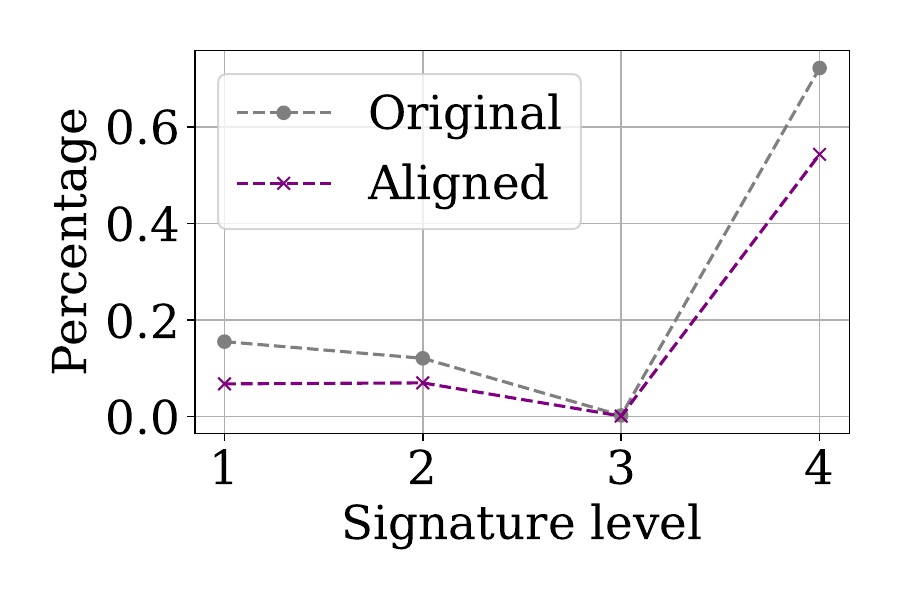}
    \caption{HPEC}
\end{subfigure}
\begin{subfigure}{0.35\textwidth}
    \centering
    \includegraphics[width=\linewidth]{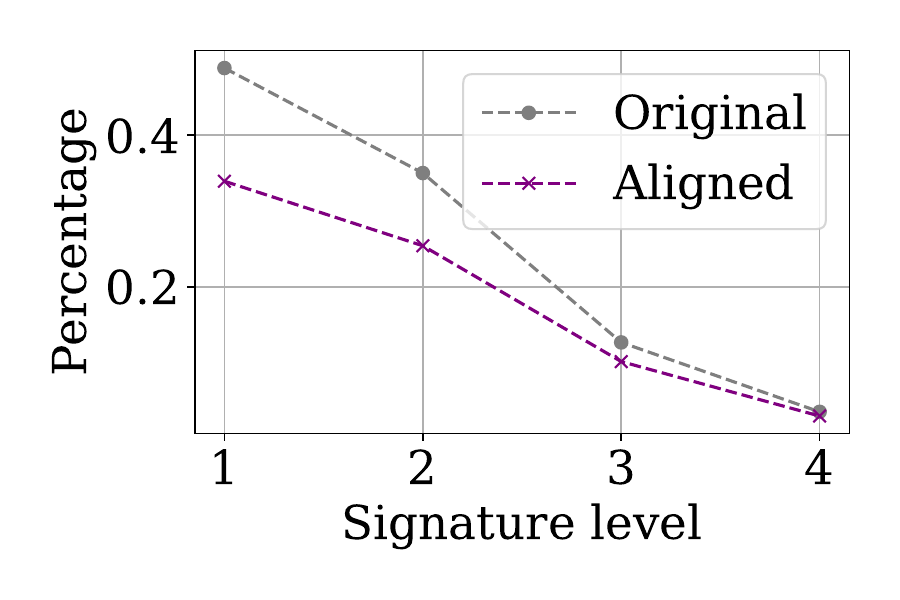}
    \caption{Exchange}
\end{subfigure}
\begin{subfigure}{0.35\textwidth}
    \centering
    \includegraphics[width=\linewidth]{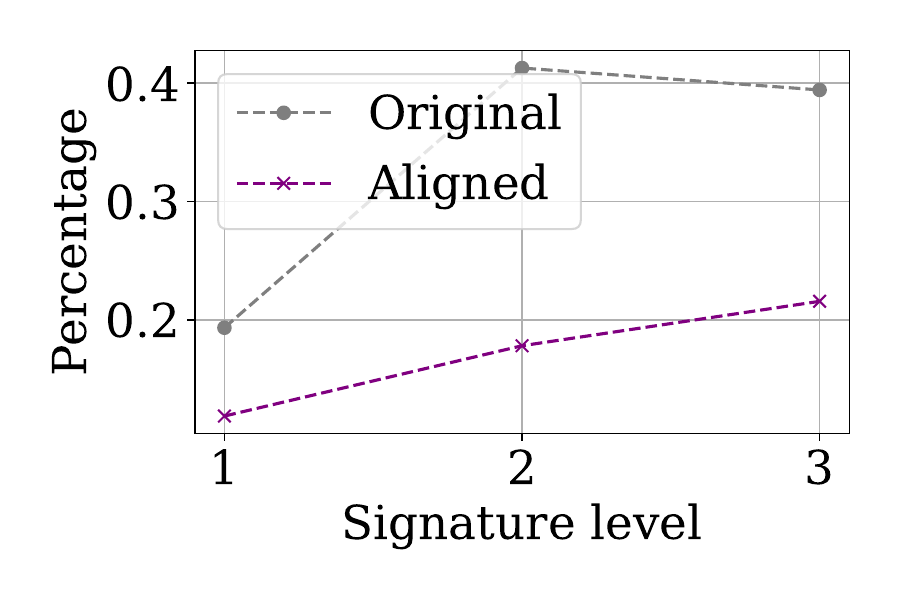}
    \caption{Weather}
\end{subfigure}
\caption{Sum of the reconstruction error averaged over the Signature level. The subcaption indicates the underlying dataset.}
\label{fig:signature_error}
\end{figure}

\section{Related work}
\label{sec:related_work}

Several works have tackled signal generation with flow-based models. We provide an overview of the main approaches in this section. Firstly, the denoising diffusion framework \citep{ho_denoising_2020} was applied to predict each future hidden state sequentially \citep{rasul_autoregressive_2021}. Latent diffusion has been used for unconditional generation of short time series \citep{lim_regular_2023} and for forecasting \citep{feng_latent_2024}. Flow matching has likewise been explored for TS tasks, such as autoregressive forecasting \citep{el-gazzar_probabilistic_2025} and conditional generation \citep{hu_flowts_2024} via rectified flows \citep{liu_flow_2022}. The flexibility of FM to allow for arbitrary source distributions has led to Gaussian processes being used as priors \citep{kollovieh_flow_2025}.

Attempts to use other TS representations include the use of seasonal and trend components \citep{cleveland_stl_1990,shen_multi-resolution_2023,yuan_diffusion-ts_2023}, the use of spectral filters as normalising flows that act on the frequency domain \citep{alaa_generative_2021}, and how the use of frequency features compares to temporal ones for diffusion-based generation \citep{crabbe_time_2024}. 
In this line, Fourier components have been used to encode time series as images to perform diffusion-based generation in that domain \citep{naiman_utilizing_2024}. Some works do use both time and other TS features, such as periodograms, to condition diffusion models for TS generation \citep{fons_lscd_2025}. Recently, \cite{khodakarami_mitigating_2026} have proposed decomposing the latent space into mean and high-frequency components, albeit to counteract the smoothing effect that spectral bias has in neural operators. Moreover, augmenting high-frequency components has been used to counteract the effect that noise addition has on these features \citep{galib_fide_2024}, and the Fourier transform has been used to complement the time-based loss of diffusion models \citep{shen_multi-resolution_2023}. Our work takes a step further in this direction.

Directly aligning representations for unconditional TS generation for more realistic local geometry in cost-efficient generation has not been addressed in the literature to the best of our knowledge.
Works about the use of other losses for training autoencoders for time series and other data modalities are presented in Appendix \ref{sec:more_losses}.

\section{Discussion}
\label{sec:analysis}
\paragraph{The convenience of spectral alignment.}
The origin of the problem of noise in latent flow matching samples can primarily be attributed to (1) a distributional discrepancy between generated data and the ground truth signals as a consequence of using flow matching with a Gaussian prior, or (2) as a result of artefacts introduced by the compression stage. The latter happens since the decoder needs to reconstruct the full signal from a lower-dimensional representation, leading to artificial elements in the final TS sample. Our alignment tackles this, which improves the final samples as demonstrated in Section~\ref{sec:wasserstein}.

Our work is not the first to integrate notions of signal processing and paths with diffusion-based generation. For instance, SigDiffusions \citep{barancikova_sigdiffusions_2024} carry out the diffusion training and sampling entirely in the space of log-signatures. While this leads to improved performance and sampling times over long time series compared to other time-based approaches, the fact that the (log)signature needs to be inverted usually leads to signals that are smoother than the dataset signals. Similarly, Diffusion-TS \citep{yuan_diffusion-ts_2023} decomposes the time series into seasonal and trend components, focusing on global rather than local structure. We believe that this may be beneficial in certain settings, but that work aiming for improved fine-grained sampling is lacking in the literature.

\paragraph{Align vs regularise.}
Despite the fact that Sobolev norms have been used as regularisers \citep{czarnecki_sobolev_2017}, we emphasise that our transform-consistency losses are not equivalent to a Sobolev regularisation of the latent space. Instead, in the event that reconstructions are smoother, these losses will encourage reconstructions to follow the regularity of the input. Regularisations only decrease the norm without following a data-based reference. Hence, our alignment procedure can be regarded as adapting the latent space so that the regularity of the data is preserved, which induces a regularisation effect as a particular case.

\paragraph{Domain adaptability.}
We can infer that the success of our spectral alignment will be more evident in datasets where the local structure plays a predominant role, and this is not well captured by the time-based training. Moreover, the criteria used to evaluate success are key. For instance, fewer benefits are observed when evaluating samples with the predictive score, as opposed to the discriminative score. Indeed, the use of neural networks may easily detect discrepancies in local structure. This suggests that spectrally-aligned samples may be better suited for data augmentation for deep learning training.

The reconstruction error location plots included in Section~\ref{sec:error_location} highlight where the alignment is most beneficial. For instance, the aligned reconstructions for the HEPC dataset show a decrease in both coarse and detailed features. By contrast, the high-level features for the exchange dataset inherently carry less of the loss share. Interestingly, the signature-alignment is able to decrease the error share at all three levels, which may explain the improvements observed for this variant. Regardless of the transform, some unwanted artefacts are still observed. This reflects on the limitations of our method, but also on the fact that minimising all these peaks would imply that the reconstruction is perfect, which is hard to achieve in the context of a compression model. Nevertheless, the fact that we can characterise the error in seemingly different ways for the same dataset shows the flexibility of our approach, and is at the core of the contribution of this paper.

\section{Conclusion}
\label{sec:conclusion}

This work focuses on improving the fidelity of synthetic time series generation through latent-space alignment. We have presented the effectiveness of tuning latent spaces to focus on fine-detail elements, as modelled by time representations beyond the waveform. Indeed, through the use of our transform-consistency losses, aligned models have shown suitability for data augmentation by improving in terms of the discriminative score, which measures indistinguishability with respect to a neural network. Moreover, the generated samples have shown an out-of-sample error lower than that of existing diffusion models for TS generation, whilst showing a significant decrease in sampling times. Potential positive impacts include improved data augmentation and benchmarking in settings where real time series data is limited or costly to collect.

We believe that this study can be extended by analysing the effect of changing the prior to other sources, such as Gaussian processes. This may further decrease the mismatch between flow samples and real signals. Furthermore, the regularity of the autoencoder (e.g., as assessed by its Lipschitz constant) may play a role in the spectral profile of the underlying samples.


\appendix
\section*{Appendix}
This appendix comprises complementary material to the main body. In Section~\ref{sec:complements_TS} we include additional background related to the wavelet and signature transforms. In Section~\ref{sec:more_losses} we provide more related work, particularly regarding latent spaces that are defined with additional loss functions. Sections~\ref{sec:proof_sobolev} and \ref{sec:wasserstein_proof} include the proofs of Proposition \ref{prop:sobolev_equivalence} and \ref{prop:wasserstein_bound}, while Section~\ref{appendix:sig_kernel} has further definitions and background that complement the connection between our Signature-consistency loss and signature kernels. Experimental details are reported in Section \ref{sec:exp_details}, and the variation of results for different Sobolev exponents is shown in Section~\ref{sec:sob_weighting}. Finally, Section~\ref{sec:signature_limitations} complements the discussion for the specific case of the signature-consistency loss and its limitations.

\section{Signal representations}
\label{sec:complements_TS}

\paragraph{Wavelet families.} The form of the wavelets $\phi$ and $\psi$ introduced in Section~\ref{sec:background_ts} will depend on the type of the particular wavelet family chosen. Common families include Haar \citep{haar_theorie_1909}, Daubechies, and Symlets, a more symmetric version of Daubechies \citep{daubechies_orthonormal_1998}.  These are orthogonal wavelet families, which implies that the signal can be reconstructed from its coefficients. The Daubechies family, in particular, is used as a general-purpose family and allows for focusing on increasing degrees of regularity in the signal, allowing for the detection of a high degree of regularity without excessive loss of time localisation \citep{mallat_wavelet_1999}. Typically, only the approximation and coarse levels of $i$ detail coefficients are kept in order to have a more compact and less noisy representation of the signal.

\paragraph{Signatures.}  
This section is based on \cite{cass_lecture_2024}. Let $x: [a,b] \to V$ be a path in $C_p([a,b], V)$, i.e., a continuous path with finite $p$-variation:
\[ \left(\sup_{D \subseteq [s,t]} \sum_{D, t_i} \| x_{t_{i+1}} - x_{t_i} \|^p \right)^{1/p} < \infty \,, \]
where $D \subseteq [s,t]$ denotes a partition of $[s,t]$. Given any subinterval $[s,t] \subseteq [a,b]$, the signature transform $S(x)_{[s,t]}$ will be given by
\[ S(x)_{[s,t]} = \left(1, S(x)^{(1)}_{[s,t]}, \ldots, S(x)^{(k)}_{[s,t]}, \ldots \right)\,, \]
where
\[ S(x)^{(k)}_{[s,t]} = \int_{s < t_1 < t_2 < \ldots < t_k < t} dx_{t_1} \otimes dx_{t_2} \otimes \ldots \otimes dx_{t_k} \,.\]
$S(x)^{(k)}_{[s,t]}$ is called the $k$-fold iterated integral.

Since the magnitude of the coefficients shows a factorial decay, it is appropriate to truncate the signature transform at a certain ``depth'' $n$. Indeed, $n$-step ST will be given by
\[
S(x)^{\leq n}_{[s,t]} = \left(1, S(x)^{(1)}_{[s,t]}, \ldots, S(x)^{(n)}_{[s,t]}, \right)
\]

Note that $S^{(k)}(x)$ lies in the $k$-fold tensor product of $\mathbb{R}^d$ \citep{cass_lecture_2024}, i.e., it has dimensionality $d^k$. Signatures, by definition, do not depend on the specific set of coordinates employed. Other practical properties of STs include invariance under reparameterisations and injectivity for certain classes of paths. We refer to \cite{cass_lecture_2024} for a formal presentation of these results.

Chen's relation is an essential result for the implementation of signatures in practice. It states that for $1 \leq p < 2$, $x \in C_p([a,b])$, $y \in C_p([b,c])$, then
\[ S(x \otimes y)_{[a,c]} = S(x)_{[a,b]} \otimes S(y)_{[b,c]} \,,\]
where $x \otimes y = \mathbf{1}_{t \in [a,b]} x_t + \mathbf{1}_{t \in [c,b]} (y_{t-b+a})$ denotes the concatenation of $x$ and $y$.

A consequence of Chen's relation is that for piecewise linear $x$:
\[ S(x)_{[t_0, t_n]} = \exp(x_{t_1} - x_{t_0}) \cdot \exp(x_{t_2} - x_{t_1}) \cdot \ldots \cdot \exp(x_{t_n} - x_{t_{n-1}}) \]
which is used to compute signatures.

In practice, one approximates the signals with piecewise linear paths. Indeed, given the output continuous piecewise linear function $f : [0,1] \rightarrow \mathbb{R}^d$ satisfying $f(t_i) = x_{t_i} \in \mathbb{R}^d$, Chen's identity \citep{lyons_differential_1998}
allows us to write:

\[ S^{( k)}(x) = \left( \int_{0 < t_1 < \dots < t_k < 1} \prod_{j=1}^k \frac{d {f_i}_j(t_j)}{dt} dt_1 \dots dt_k \right)_{i_1, \dots, i_k \leq d} \,, \]
which follows the notation of \cite{kidger_deep_2019}.

\section{Additional loss functions for AEs}
\label{sec:more_losses}

A wide range of variations on the training strategy have been proposed to improve representations. In the context of image generation, latent spaces are often trained using perceptual losses \citep{dieleman_generative_2025}. Indeed, the availability of pre-trained networks having an outstanding correlation with human perception \citep{zhang_unreasonable_2018} has provided the community with a useful toolkit to model high-level features. These perceptual losses have proved effective to use in the context of image autoencoders \citep{pihlgren_improving_2020}.

Moreover, perceptual features have been used to ensure the consistency between input and output images along the intermediate hidden representations \citep{hou_deep_2017}. In the context of video generation, additional losses have been implemented to train the autoencoder \citep{hacohen_ltx-video_2024}, enforcing the consistency of the L1 distance between the discrete wavelet transforms of the input and the reconstruction.

In order to improve time series representations, other features have been leveraged in the context of contrastive learning \citep{trirat_universal_2024}. For instance, frequency invariance is improved by minimising the distance between frequency embeddings, as well as between time embeddings \citep{zhang_self-supervised_2022}. This is done for samples and between samples and augmentations (positive pairs for contrastive learning), which aims at obtaining representations that are generalizable to unseen datasets via fine-tuning. Time-frequency invariance has also been used for TS representations that allow for anomaly detection \citep{li_time-frequency_2026}. On the other hand, TS2Vec learns representations of each timestep contrastively, by sampling overlapping segments of a TS and matching the vectors within the common subsegment \citep{yue_ts2vec_2022}. Crucially, their temporal and instance-wise contrastive losses are computed at multiple granularity levels via iterative max pooling. In addition, learnt shapelets 
have also been used as features in multi-scale TS representation learning \citep{liang_shapelet-based_2023}. These works have, however, focused on universal representations that can lead to performance in downstream tasks, rather than tackling generative modelling.

\section{Proof of Proposition \ref{prop:sobolev_equivalence}}
\label{sec:proof_sobolev}

By definition of the Sobolev norm and using the fact that $B_k$, $k=1,\dots,K$ partitions the frequency domain,
\[ \| x - \bar{x} \|_{H^s} = \sum^K_{k=1} \sum_{\xi_k\in B_k} (1+|\xi|^2)^{s} |\mathcal{F}[x-\bar{x}](\xi_k)|^2 \,.\]
Since, by construction, we have that
\[ \sum_{\xi_k\in B_k} m^2_k |\mathcal{F}[x-\bar{x}](\xi_k)|^2 \leq \sum_{\xi_k\in B_k} (1+|\xi|^2)^s |\mathcal{F}[x-\bar{x}](\xi_k)| \,,\]
using the linearity of $\mathcal{F}$, we obtain that
\[ \sum^K_{k=1} \sum_{\xi_k\in B_k} m^2_k |\mathcal{F}[x](\xi_k)-\mathcal{F}[\bar{x}](\xi_k)|^2 \leq \| x - \bar{x} \|_{H^s}  \,.\]
The inequality $\| x - \bar{x} \|_{H^s} \leq \sum^K_{k=1} \sum_{\xi_k\in B_k} M^2_k |\mathcal{F}[x](\xi_k)-\mathcal{F}[\bar{x}](\xi_k)|^2 $ follows analogously. We conclude that $\mathcal{L}_{\mathcal{F},s}$ is comparable to $\| \cdot \|_{H^s}$.

\section{Proof of Proposition \ref{prop:wasserstein_bound}}
\label{sec:wasserstein_proof}

\begin{proof}
By the triangle inequality, we have:
\[ W_2^{(T)}(p_\text{data}, \tilde{p}_1) \leq W_2^{(T)}(p_\text{data}, \bar{p}_1)+ W_2^{(T)}(\bar{p}_1, \tilde{p}_1) \]
for $\bar p_1 = (\mathcal{D} \circ \mathcal{E})_\# p_\text{data}$.

On the one hand, since the Wasserstein distance takes the minimum value of $\int d_T(x_1,x_2)^2 \pi(x_1,x_2)$, in particular it holds that for the coupling given by taking $x\sim p_\text{data}$ and $\tilde x = \mathcal{D}(\mathcal{E}(x))$:
$W_2^{(T)}(p_1, \bar{p}_1) \leq \mathbb{E}_{x\sim p_\text{data}}[d_T(x_1, \tilde{x})]^{1/2}$,
which corresponds to the expected value of the transform-consistency loss.

For the second term of the right-hand side, we have
\[W_2^{(T)}(\bar{p}_1, \tilde{p}_1) \leq L_{\mathcal{D},T} \, W_2(p_{\tilde{z}}, p_z) \leq L_{\mathcal{D},T} \left( e^{1+2L_\theta} H(u_\theta) \right)^{1/2} \]
where we have used \cite{albergo_building_2022} for the second inequality, by assuming the flow and decoder to be Lipschitz.
\end{proof}

\section{The signature-consistency loss as a kernel}
\label{appendix:sig_kernel}

In this section, we highlight the connection between the signature-consistency loss and signature kernels.

\subsection{Unparameterised paths}
\label{sec:unparameterised_paths}
This signature kernel, as introduced in Eq.~\ref{eq:sig_kernel} will be well-defined if
\[ \|S(x)\|_w^2 = \sum_{k=0}^{\infty} w_k\, \|S(x)^{(k)}\|_{V^{\otimes k}}^2 < \infty, \]
which holds when $\sum_{k=0}^\infty \frac{C^k w_k}{(k!)^2}$ is a convergent series for any $C > 0$.

Denote by $\sim_t$ the tree-like equivalence relation for absolutely continuous paths, i.e., $x$ and $y$ will be tree-like equivalent if there exists a time reparametrisation $\tau(\cdot)$ such that $x(t) = y(\tau(t))$. Since $S(x) = S(y)$ if and only if $x \sim_t y$ \citep{boedihardjo_signature_2016}, the signature kernel will be defined for elements of the quotient space $C_{0,p}/\sim_t$, where $C_{0,p}$ is the subspace of $C_p(V)$ such that paths start at $0$. This space is often referred to as the set of unparameterised paths, and its elements are denoted by $[x]$.

Note that, even though $k_w$ is defined for unparameterised paths, it can be extended to continuous paths with bounded $1$-variation \citep{cass_lecture_2024}. Consequently, the notation $k_w(x,y)$ can be adopted.

\subsection{RKHS and the reproducing property}
\label{sec:RKHS}
This section is based on \cite{cass_lecture_2024}. Recall that a Hilbert space of functions defined over $\mathcal{X}$ is said to be a reproducing kernel Hilbert space (RKHS) if $f \mapsto f(x)$ is a continuous linear functional, that is, $\exists\, C_x \geq 0$ such that
\[
|f(x)| \leq C_x \|f\|_{\mathcal{H}}, \qquad \forall f \in \mathcal{H}.
\]
The Riesz representation theorem guarantees the existence of a unique functional $k(x, \cdot)$ in the RKHS $\mathcal{H}$ satisfying
\begin{equation}
\langle k(x, \cdot), f \rangle_{\mathcal{H}} = f(x), \qquad \forall f \in \mathcal{H}, \; \forall x \in \mathcal{X}.
\label{eq:reproducing}
\end{equation}

Equation~\eqref{eq:reproducing} is referred to as the reproducing property. This then defines a kernel
\[ k(x,y) = \langle k(x,\cdot), k(y,\cdot) \rangle_{\mathcal{H}}. \]

Conversely, thanks to the Moore-Aronszajn theorem \citep{aronszajn_theory_1950}, given a positive semi-definite kernel, i.e., such that $k := (k(x_i,x_j))_{i,j}$ is a positive semi-definite matrix for any $n \in \mathbb{N}$ and any $x_1,\ldots,x_n \in \mathcal{X}$, there exists a unique RKHS $\mathcal{H}$ such that the reproducing property in Equation~\eqref{eq:reproducing} holds.

Signatures can be used to define such spaces. Given $v = (v_1,\ldots,v_k)$ and $\tilde{v} = (\tilde{v}_1,\ldots,\tilde{v}_k)$ in $V^{\otimes k}$ and a weight function $w : \mathbb{N} \cup \{0\} \to \mathbb{R}_+$, the $w$-inner product $\langle \cdot, \cdot \rangle_w$ over the tensor algebra $T(V)$ is given by
\[ \langle v, \tilde{v} \rangle_w := \sum_{k=0}^{\infty} w_k\, \langle v_k, \tilde{v}_k \rangle_{V^{\otimes k}}, \]
where $\langle v_k, \tilde{v}_k \rangle_{V^{\otimes k}}$ is the canonical Hilbert-Schmidt inner product
\[ \langle v, \tilde{v} \rangle_{V^{\otimes k}} = \prod_{i=1}^k \langle v_i, \tilde{v}_i \rangle_V. \]

\subsection{Proof of proposition \ref{prop:sigkernel}}
\label{sec:sigkernel_proof}

\begin{proof}
Let $\tilde{w}_k = (1+a^2k^2)^s$ for $a>0$ and $s>0$, and let $C>0$. The series $\sum_{k=0}^\infty \frac{C^k\tilde{w}_k}{(k!)^2}$ is known to converge. Indeed, the ratio test implies that if
\[ \lim_{k \to \infty} \left| \frac{ C^{k+1}\tilde{w}(k+1)}{(k+1)!} \cdot \frac{k!}{C^{k}\tilde{w}_k} \right| < 1, \]
then the series converges absolutely. Since
\begin{alignat*}{2}
\frac{ C^{k+1} \tilde{w}(k+1)}{(k+1)!} \frac{k!}{C^{k}\tilde{w}_k} & = \frac{C}{(k+1)^2}\left(\frac{1+a^2(k+1)^2}{1+a^2k^2} \right)^s
\\ & = \frac{C}{(k+1)^2} \left( 1+\frac{2k+1}{k^2} \right)^s 
\\ & \longrightarrow_{k\to\infty} 0,
\end{alignat*}
this is indeed the case. Then, for any $ \sum_{k=0}^{\infty} \frac{C^k \tilde w_k}{(k!)^2}$
converges absolutely. This, in particular, holds for $a=\frac{1}{2 k_{\text{max}}}$. Consequently, by Lemma 2.1.10 in \cite{cass_lecture_2024}, $k_w$ is a positive semi-definite kernel. Then the Moore-Aronszajn theorem \citep{aronszajn_theory_1950} implies that there is a unique RKHS such that $k_w$ has the reproducing property.
\end{proof}

\section{Datasets and experimental details}
\label{sec:exp_details}

Following the experimental setup described in Section \ref{sec:experiments}, we provide details on the datasets, running times and parameters in Table \ref{tab:exp_details}. Times are reported using an NVIDIA GeForce RTX3090 GPU. The hidden dimensionality of the latent space was set to $64$ for high-dimensional sets, and $32$ for HEPC. The last linear layer of the encoder projects the output of the convolutions to a latent space of $128$ dimensions ($64$ in the case of HEPC). The base autoencoders were trained for $500$ epochs. While training the AE for considerably longer yielded cleaner samples, the spectral alignment achieves this sample quality with a very quick fine-tuning stage as opposed to relying on longer training times.


\begin{table*}[!ht]
\centering
\caption{Experimental details for the empirical evaluation of the transform-consistency losses. We present each dataset and its dimensionality. This, in turns, modifies the dimensionality of their transformed inputs, which are detailed along with the depths/decomposition levels used. The training times for each stage are also provided.}
\label{tab:exp_details}
\small{
\begin{tabular}{@{}lcllccccc@{}}
\toprule
Dataset & Dim. & N. train & Variant & Levels (dim.) & Train time & Fine-tuning & \multicolumn{1}{l}{flow training} & Total time \\ \midrule
\multirow{4}{*}{HEPC} & \multirow{4}{*}{1} & \multirow{4}{*}{8242} & Base &  & 4m 5s &  & \multirow{4}{*}{5m 34s} & 9m 39s \\
 &  &  & Fourier & - (1000) &  & 14s &  & 9m 53s \\
 &  &  & Wavelet & 5 (1040) &  & 34s &  & 10m 13s \\
 &  &  & Signature & 5 (62) &  & 6m 35s &  & 16m 14s \\ \midrule
\multirow{4}{*}{\begin{tabular}[c]{@{}l@{}}Exchange\\ rates\end{tabular}} & \multirow{4}{*}{8} & \multirow{4}{*}{5088} & Base &  & 3m 2s &  & \multirow{4}{*}{3m 24s} & 6m 26s \\
 &  &  & Fourier & - (1000 $\times$ 8) &  & 9s &  & 6m 35s \\
 &  &  & Wavelet & 5 (1040 $\times$ 8) &  & 29s &  & 6m 55s \\
 &  &  & Signature & 3 (584) &  & 3m 54s &  & 10m 20s \\ \midrule
\multirow{4}{*}{Weather} & \multirow{4}{*}{14} & \multirow{4}{*}{8340} & Base &  & 5m 12s &  & \multirow{4}{*}{2m 52s} & 8m 4s \\
 &  &  & Fourier & - (1000 $\times$ 14) &  & 17s &  & 8m 21s \\
 &  &  & Wavelet & 5 (1040 $\times$ 14) &  & 37s &  & 8m 21s \\
 &  &  & Signature & 3 (2954) &  & 3m 6s &  & 11m 10s \\ \bottomrule
\end{tabular}}
\end{table*}

Model details are provided in Appendix~\ref{sec:exp_details}, Table~\ref{tab:model_details}. The encoder-decoder pair architecture corresponds to mirrored convolutional neural networks with a stride of 2, which halves the temporal resolution after each of the two layers. The hidden dimensionality is set to half the latent space dimensionality, and the models were trained for 500 epochs with a learning rate of 0.0001. The velocity field of the flow matching model is parameterised via a UNET whose blocks are multilayer perceptrons. The learning rate used is 0.001, and the time is included using a learnable embedding of the same dimensionality as the latent codes. The batch size was set to 128 for both latent space and flow training.

\begin{table*}[!ht]
\centering
\caption{Specifications of epochs and latent and hidden dimensionality for each dataset. The abbreviation ``dim.'' stands for dimensionality and ``params.'' for parameters. The number of epochs used to train the laten flow models is shown as ``FM epochs''.}
\label{tab:model_details}
\small{
\begin{tabular}{@{}lccccc@{}}
\toprule
Dataset & Dim. & \multicolumn{1}{l}{Latent dim.} & \multicolumn{1}{l}{Inner dim. flow} & \multicolumn{1}{l}{Total params.} & \multicolumn{1}{l}{FM epochs} \\ \midrule
HEPC & 1 & 64 & 258 & 2.021 M & 500 \\ \midrule
\begin{tabular}[c]{@{}l@{}}Exchange\\ rates\end{tabular} & 8 & 128 & 512 & 8.037 M & 500 \\ \midrule
Weather & 14 & 128 & 512 & 8.037 M & 250 \\ \bottomrule
\end{tabular}
}
\end{table*}

\section{The exponent of the Sobolev weighting}
\label{sec:sob_weighting}

Our weighting scheme is inspired by the theoretical connections with the Sobolev norms outlined in Section \ref{sec:sobolev}. Throughout our experiments, we set the Sobolev exponent to $1$. As can be seen in Figure \ref{fig:sob_variation}, only a slight increase in the marginal score, which corresponds to the absolute difference between the real and generated marginal distributions, can be seen as $s$ increases. By contrast, while the discriminative score does not vary significantly for the HEPC and exchange rates datasets, less improvement is observed for $s=2.5$ for the linear transforms applied to the Weather dataset. Overall, we propose $s=1$ as the standard choice. The performance of the signature variants for HEPC and exchange datasets is discussed in Section \ref{sec:signature_limitations}.

\begin{figure*}[h]  
\centering
\begin{subfigure}{0.48\textwidth}
    \centering
    \includegraphics[width=\linewidth]{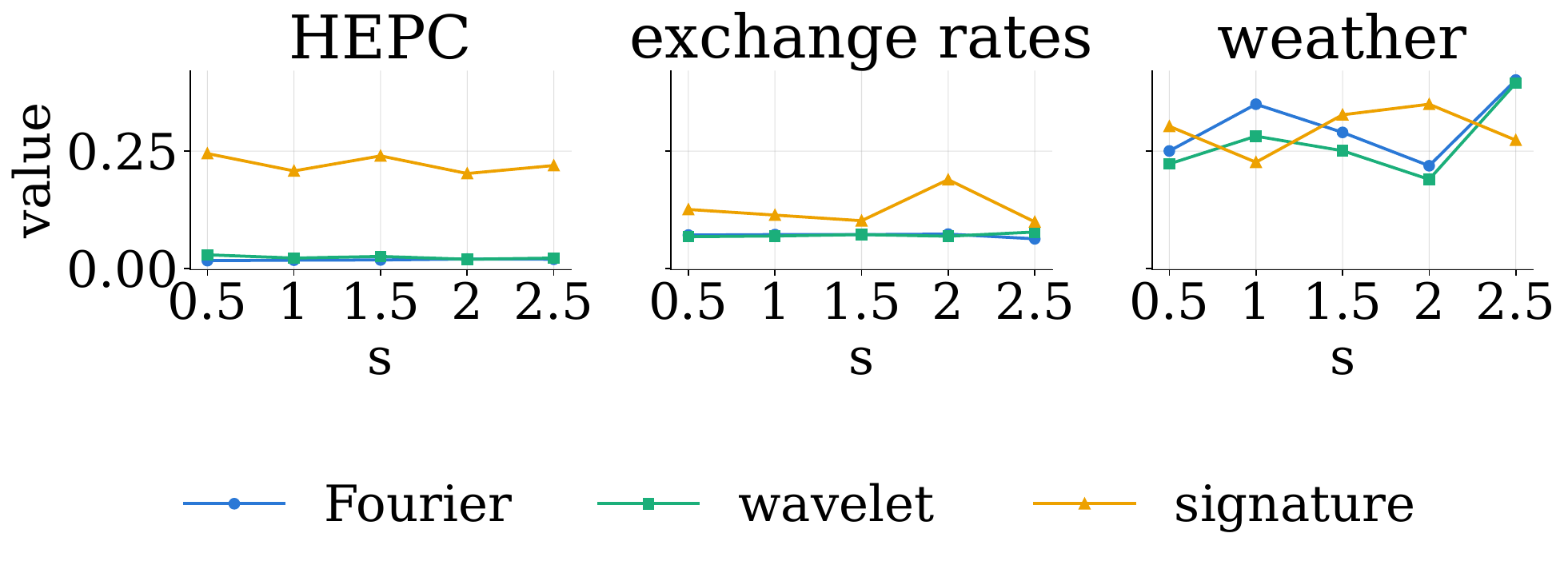}
    \caption{Discriminative score}
\end{subfigure}
\begin{subfigure}{0.48\textwidth}
    \centering
    \includegraphics[width=\linewidth]{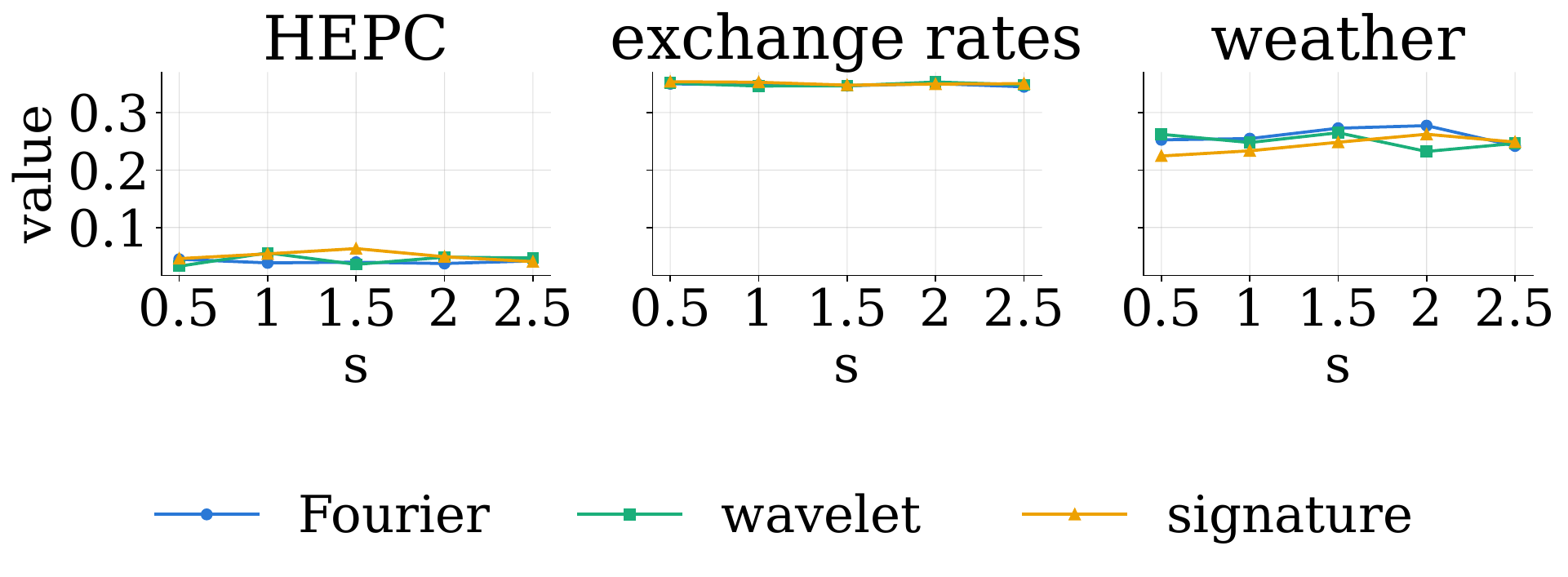}
    \caption{Marginal score}
\end{subfigure}
\caption{Effect of varying the Sobolev exponent $s$.}
\label{fig:sob_variation}
\end{figure*}

\section{Remarks on the use of signatures}
\label{sec:signature_limitations}

The signature transform is unique with respect to the Fourier and Wavelet transforms, in the sense that it is an infinite series of coefficients. High-dimensional series may prove too costly for the use of the signature transform as a transform-consistency loss when intended to be used with more depth in high-dimensional datasets. For instance, for the Weather dataset (of dimensionality 14), we are only able to use up to a signature depth of 3 before this representation becomes too large for a consumer-level GPU. While the dimensionality of log-signatures is lower than that of signatures, their computation in practice relies on computing full signatures first, thus not alleviating the GPU memory problem. As reported in Appendix \ref{sec:exp_details}, fine-tuning with the signature loss entails a significantly higher cost when compared to using linear transforms such as Fourier or wavelet, linked to their computation (and their gradients) more than to the dimensionality of the vectors themselves.

The resolution-free nature of the signature transform has the advantage that its dimension does not increase as the number of measurements of a signal increases, but it also means that the signature may end up being an under-representation of the signal and not capture the high-level behaviour of it if the depth is not enough. While the signature alignment presents a similar improvement to that of the linear transforms for the weather dataset ($14$ dimensions), it presents less improvement for the HEPC and exchange rates datasets. For the former, as shown in Table \ref{tab:sig_depth}, for which the original signals are one-dimensional, applying the signature transform to the time-augmented path leads to dimensionalities that can be insufficient to characterise the local behaviour. The performance does improve for depth 7, but the long alignment times that this entails make it impractical. While this is a limitation of this particular variant, the signature alignment might still prove useful for signals with high dimensionality, such as the weather dataset, where even a depth of 3 encodes sufficiently useful information for alignment.

\begin{table}[!ht]
\centering
\caption{Change in fine-tuning times (25 epochs), dimensionality, and discriminative score for varying signature transform depth as observed for the HEPC dataset. The abbreviation disc. stands for the discriminative score, and the 0 before the decimal is omitted from the discriminative score results.}
\label{tab:sig_depth}
\small{
\begin{tabular}{@{}cccccc@{}}
\toprule
depth & 3 & 4 & 5 & 6 & 7 \\ \midrule
time & 2m 33s & 4m 15s & 6m 35s & 9m 29s & 12m 50s \\
dim & 14 & 30 & 62 & 126 & 254 \\
disc. & .21±.13 & .17±.12 & .18±.14 & .17±.12 & .12±.11 \\ \bottomrule
\end{tabular}
}
\end{table}


\printcredits

\bibliographystyle{cas-model2-names}

\bibliography{references}



\end{document}